\documentclass[11pt]{article}
\usepackage{geometry}
\usepackage{microtype}
\usepackage{subcaption}
\usepackage{xspace}
\usepackage{booktabs} 
\pdfoutput=1
\usepackage{hyperref}
\usepackage[toc,page]{appendix}
\usepackage{amsfonts,bm}

\newcommand{\vol}{{\hbox{\textnormal{vol}}}}

\newcommand{\ex}{{\hbox{ex}}}
\newcommand{\supp}{{\hbox{\textnormal{supp}}}}
\newcommand{\flow}{{\hbox{\textnormal{flow}}}}
\newcommand{\sign}{\mathop{\mathrm{sign}}}
\newcommand{\argmin}{\mathop{\mathrm{argmin}}}

\newcommand{\ddeg}{\vec{\boldsymbol{\mathit{d}}}}
\newcommand\Otil{\widetilde{O}}

\usepackage[utf8]{inputenc}
\usepackage[numbers]{natbib}
\usepackage{amsmath}
\usepackage{amsthm}
\usepackage{adjustbox}
\usepackage{rotating,multirow}
\usepackage{arydshln}
\usepackage{algorithm}
\usepackage{algorithmic}
\usepackage{enumitem}

\usepackage{cleveref}

\usepackage{apptools}
\AtAppendix{\counterwithin{theorem}{section}} 

\newtheorem{theorem}{Theorem}
\newtheorem{proposition}{Proposition}
\newtheorem{corollary}[theorem]{Corollary}
\newtheorem{lemma}[theorem]{Lemma}

\newtheorem{claim}{Claim}

\newtheorem{assumption}{Assumption}

\def\norm#1{\left\| #1 \right\|}

\begin{document}

\title{$p$-Norm Flow Diffusion for Local Graph Clustering\textsuperscript{}\thanks{Equal contribution. Code is available at \url{http://github.com/s-h-yang/pNormFlowDiffusion}.}}

\author{\name Shenghao~Yang\textsuperscript{ \hspace{0.1cm}} \email shenghao.yang@uwaterloo.ca \\ 
        \addr School of Computer Science, \\ University of Waterloo, \\ Waterloo, ON, Canada.
        \AND
        \name Di~Wang\textsuperscript{ \hspace{0.1cm}} \email wadi@google.com \\ 
        \addr Google, \\ Mountain View, CA, US.
        \AND
        \name Kimon~Fountoulakis\textsuperscript{ \hspace{0.1cm}} \email kfountou@uwaterloo.ca \\ 
        \addr School of Computer Science, \\ University of Waterloo, \\ Waterloo, ON, Canada.
}

\author{
        Kimon~Fountoulakis%
        \thanks{School of Computer Science, University of Waterloo, Waterloo, ON, Canada. E-mail: kfountou@uwaterloo.ca.
                KF would like to acknowledge NSERC for providing partial support for this work.
        }
        \and
        Di~Wang%
        \thanks{Google Research, Mountain View, CA, USA. E-mail: wadi@google.com.
        } 
        \and
	Shenghao~Yang%
        \thanks{School of Computer Science, University of Waterloo, Waterloo, ON, Canada. E-mail: s286yang@uwaterloo.ca.
        }
}

\maketitle

\begin{abstract}
\fontsize{10pt}{12pt}\selectfont
Local graph clustering and the closely related seed set expansion problem are primitives on graphs that are central to a wide range of analytic and learning tasks such as local clustering, community detection, semi-supervised learning, nodes ranking and feature inference. Prior work on local graph clustering mostly falls into two categories with numerical and combinatorial roots respectively.\ In this work, we draw inspiration from both fields and propose a family of convex optimization formulations based on the idea of diffusion with $p$-norm network flow for $p\in (1,\infty)$.

In the context of local clustering, we characterize the optimal solutions for these optimization problems and show their usefulness in finding low conductance cuts around input seed set.\ In particular, we achieve quadratic approximation of conductance in the case of $p=2$ similar to the Cheeger-type bounds of spectral methods, constant factor approximation when $p\rightarrow\infty$ similar to max-flow based methods, and a smooth transition for general $p$ values in between.\ Thus, our optimization formulation can be viewed as bridging the numerical and combinatorial approaches, and we can achieve the best of both worlds in terms of speed and noise robustness. We show that the proposed problem can be solved in strongly local running time for $p\ge 2$ and conduct empirical evaluations on both synthetic and real-world graphs to illustrate our approach compares favorably with existing methods.\

\end{abstract}

\section{Introduction}
Graphs are ubiquitous when it comes to modeling relationships among entities, e.g. social~\cite{TMP2012} and biology~\cite{Tuncbag-2016-glioblastoma} networks, and graphs arising from modern applications are massive yet rich in small-scale local structures~\citep{LLDM09_communities_IM,Jeub15,FM16}.\ Exploiting such local structures is of central importance in many areas of machine learning and applied mathematics, e.g. community detection in networks~\cite{NJW01_spectral,WS05_spectralSDM,LLDM09_communities_IM,Jeub15} and PageRank-based spectral ranking in web ranking \cite{PB99,Gle15_SIREV}.
 Somewhat more formally, we consider local graph clustering as the task of finding a community-like cluster around a given set of seed nodes, where nodes in the cluster are densely connected to each other while relatively isolated to the rest of the graph.\
 Moreover, an algorithm is called {\em strongly local} if it runs in time proportional to the size of the output cluster rather than the size of the whole graph.\ 
 
Strongly local algorithms for local graph clustering are predominantly based on the idea of diffusion, which is the generic process of spreading mass among vertices by sending mass along edges.\ 
 This connection has been formalized in many previous results~\citep{ST13,ACL06,WFHM2017}.\
Historically, the most popular diffusion methods are spectral methods (e.g. random walks) based on the connection between graph structures and the algebraic properties of the spectrum of matrices associated with the graph~\cite{LS90,LS93,chung2007-pagerank-heat,Tsiatas12}.
Spectral diffusion methods are widely applied in practice due to the ease of implementation, efficient running time and good performances in many contexts. However, it is also known in theory and in practice that spectral methods can spread mass too aggressively and not find the right cluster when structural heterogeneities exist, and thus are not robust on real-world graphs constructed from very noisy
data~\cite{guatterymiller98,luxburg05_survey,LLDM09_communities_IM,Jeub15}. 
A more recent line of work on diffusion is based on the combinatorial idea of max flow exploiting the canonical duality between flow and cut~\cite{WFHM2017,FDM2017,OZ14}.\ These methods offer improved theoretical guarantees in terms of locating local cuts, and have been shown to be important for pathological cases where spectral methods do not perform well in practice.\ However, combinatorial methods are generally accepted to be more difficult to understand and implement in practice due to the more complicated underlying dynamics.

In this paper, we propose and study a family of primal and dual convex optimization problems for local graph clustering. We call the primal problem $p$-norm flow diffusion, parameterized by the $\ell_p$-norm used in the objective function, and the problem defines a natural diffusion model on graphs using network flow. We refer the dual problem as the $q$-norm local cut problem where $q$ is the dual norm of $p$ (i.e. $1/p+1/q=1$). The optimal solution to the $q$-norm local cut problem can be used to find good local clusters with provable guarantees. Throughout our discussion, we use $p$-norm diffusion to refer both the primal and the dual problems since our main inspiration comes from diffusion, even though our results are technically for the $q$-norm local cut. 

We note that almost all previous diffusion methods are defined with the dynamics of the underlying diffusion procedure, i.e. the step-by-step rules of how to send mass,
 and the analysis of these methods is based on the behaviors of the algorithm.\ Although there are some efforts on interpreting the algorithms with optimization objectives~\citep{FDM2017,FKSCM2017,MOV12_JMLR}, this line of research remains predominantly bottom-up starting with the algorithmic operations.\ On the other hand, our work starts with a clear optimization objective, and analyze the properties of the optimal solution independent from what method is used to solve the problem.\ This top-down approach is distinct in theory, and is also very valuable in practice since the de-coupling of objective and algorithm gives the users the freedom at implementation to choose the most suitable solver based on availability of infrastructure and code-base.\


\subsection{Our Main Contributions}
We refer readers to Section~\ref{sec:notations} for formal discussion. Our first main result is a novel theoretical analysis for local graph clustering using the optimal solution of $p$-norm diffusion.\ In particular, suppose there exists a cluster $B$ with conductance $\phi(B)$, and we are given a set of seed nodes that overlaps reasonably with $B$.\ Then the optimal solution of $p$-norm diffusion can be used to find a cluster $A$ with conductance at most $\mathcal{O}(\phi(B)^{1/q})$.\ For $p=2$, this result resembles the Cheeger-type quadratic guarantees that are well-known in spectral-based local graph clustering literature~\cite{ST13,ACL06}.\
When $p\rightarrow\infty$, our conductance guarantee approaches a constant factor approximation similar to max flow methods, while achieving a smooth transition for general $p$ values in between.\ We observe in practice that our optimization formulation can achieve the best of both worlds in terms of speed and robustness to noise for when $p$ lies in the range of small constants, e.g. $p=4$.

On the algorithm side, we show that a randomized coordinate descent method can obtain an $\epsilon$ accurate solution of $p$-norm diffusion for $p\ge 2$ in strongly local running time.\ 
The running time of our algorithm captures the effect that it takes longer for the algorithm to converge for larger $p$ values.\ 
Although the iteration complexity analysis is not entirely new, we show a crucial result on the monotonicity of the dual variables, which establishes the strongly local running time of the algorithm.\

Our analysis illustrates a natural trade-off as a function of $p$ between robustness to noise and the running time for solving the dual of the $p$-norm diffusion problem.\ In particular, for $p=2$ the diffusion problem can be solved in time linear in the size of the local cluster
, but may have quadratic approximation error $\mathcal{O}(\sqrt{\phi(B)})$.\ On the other hand, the approximation error guarantee improves when $p$ increases, but it also takes longer to converge to the optimal solution. We believe the regime of $p$ being small constants offer the best trade-offs in general.

\subsection{Previous Work}


The local clustering problem is first proposed and studied by~\citep{ST13}.\ Their algorithm Nibble is a truncated power method with early termination, and
their result were later improved to $\tilde{\mathcal{O}}(\sqrt{\phi(B)})$ conductance approximation  and $\tilde{\mathcal{O}}\left(\frac{\vol(B)}{\phi(B)}\right)$ time using 
approximate personalized PageRank~\citep{ACL06}, which is one of the most popular local clustering methods.\ The EvoCut algorithm~\citep{AP09} is the fastest spectral based local clustering method with running time $\tilde{\mathcal{O}}\Big(\frac{\vol(B)}{\sqrt{\phi(B)}}\Big)$.\
In~\citep{ALM13} the authors analyzed the behavior of approximate personalized PageRank under certain intra-cluster well-connected conditions to give strengthened results.\
There are many other spectral based diffusions, examples include local Lanczos spectral approximations \citep{2017-ecml-pkdd}, evolving sets \citep{AGPT2016}, seed expansion methods~\citep{KK2014} and heat-kernel PageRank~\citep{C09,KG14}.\
Note that the above methods based on spectral diffusion are subject to the quadratic approximation error in worst case, informally known as the Cheeger barrier.\

Combinatorial methods for local graph clustering are mostly based on the idea of max flow. Some examples include the flow-improve method~\citep{AL08_SODA}, the local flow-improve method~\citep{OZ14} and the capacity releasing diffusion~\citep{WFHM2017}.\
The former relies on black-box max flow primitives, while the latter two require specialized max flow algorithms with early termination to to achieve running time that is nearly linear in the size of the local cluster.\ These algorithms achieve constant approximation error, i.e., the output cluster has conductance $\mathcal{O}(\phi(B))$, as opposed to quadratic approximation error for spectral methods.\

The first graph clustering method to explicitly incorporate norms beyond $p=2$ is the $p$-spectral clustering proposed by~\citep{BM2009}, based on the notion of graph $p$-Laplacians initially studied in~\citep{Amghibech2003}. It generalizes the standard spectral approach for global graph partitioning and achieves tighter approximations beyond the Cheeger barrier. Similar ideas are later extended to the context of hypergraphs~\citep{LM2018a}. All these methods rely on global analysis and computation of eigenvalues and eigenvectors, and thus do not enjoy the same properties that local methods have.

\section{Preliminaries and Notations}
\label{sec:notations}
We consider un-directed connected graph $G$ with $V$ being the set of nodes and $E$ the set of edges. For simplicity we focus on un-weighted graphs in our discussion, although our result extends to the weighted case in a straightforward manner. 
The degree $\deg(v)$ of a node $v\in V$ is the number of edges incident to it, and we denote $\ddeg$ as the vector of all nodes' degrees and $D=\mbox{diag}(\ddeg)$. We refer to $\vol(C)=\sum_{v\in C}\deg(v)$ as the {\em volume} of $C\subseteq V$. We use subscripts to indicate what graph we are working with, while we omit the subscripts when the graph is clear from context. 

A {\em cut} is treated as a subset $S\subset V$, or a partition $(S, \overline{S})$ where $\overline{S} = V\setminus S$. For any subsets $S,T\subset V$, we denote $E(S,T)=\{\{u,v\}\in E\mid u\in S,v\in T\}$ as the set of edges between $S$ and $T$. The {\em cut-size} of a cut $S$ is $\delta(S)=|E(S,\overline{S})|$. The {\em conductance} of a cut $S$ in $G$ is $\Phi_G(S)= \frac{\delta(S)}{\min(\vol_G(S),\vol_G(V\setminus S))}$. Unless otherwise noted, when speaking of the conductance of a cut $S$, we assume $S$ to be the side of smaller volume. The conductance of a graph $G$ is $\Phi_{G}=\min_{S\subset V}\Phi_G(S)$. 

A {\em routing} (or {\em flow}) is a function $f: E \rightarrow \mathbb{R}$. For each un-directed edge $e$ with endpoints $u,v$, we arbitrarily orient the edge as $e=(u,v)$, i.e. from $u$ to $v$. The magnitude of the flow over $e$ specifies the amount of mass routed over $e$, and the sign indicates whether we send flow in the forward or reverse direction of the orientation of edge $e$, i.e. $f(u,v)$ is positive if mass is sent from $u$ to $v$ and vice versa. We abuse the notation to also use $f(v,u)=-f(u,v)$ for an edge $e=(u,v)$. We denote $B$ as the signed incidence matrix of the graph of size $|E|\times |V|$ where the row of edge $e=(u,v)$ (again, using the arbitrary orientation) has two non-zeros entries, $-1$ at column $u$ and $1$ at column $v$. Throughout our discussion we refer to a function over edges (or nodes) and its explicit representation as an $|E|$-dimensional (or $|V|$-dimensional) vector interchangeably.




\section{Diffusion as Optimization}
\label{sec:lpdiffusion}
Given a graph $G$ with signed incidence matrix $B$, and two functions $\Delta, T:V\rightarrow \mathbb{R}_{\geq 0}$, we propose the following pair of convex optimization problems, which are the $p$-norm flow diffusion
\begin{equation}
    \begin{aligned}
        \mbox{minimize} & \ \|f\|_p \\
        \mbox{s.t.  } & \ B^Tf + \Delta \le T
    \end{aligned}
    \label{eq:primalLp}
\end{equation}
and its dual formulation with $q$ such that $1/q+1/p=1$
\begin{equation}
    \begin{aligned}
        \mbox{maximize} & \ (\Delta - T)^Tx \\
        \mbox{s.t.  } & \ \norm{Bx}_q \le 1 \\
        & \ x \ge 0.
    \end{aligned}
    \label{eq:dualLq}
\end{equation}
The solution to the dual problem $x\in \mathbb{R}^{|V|}_{\geq 0}$ gives an embedding of the nodes on the (non-negative) real line. This embedding is what we actually compute in the context of local clustering, and we use the primal problem and its flow solution $f\in \mathbb{R}^{|E|}$ mostly for insights and analysis purposes. We discuss the interpretation of the primal problem as a diffusion next. 

\paragraph{The Primal Problem. } As mentioned earlier, we consider a diffusion on a graph $G=(V,E)$ as the task of spreading mass from a small set of nodes to a larger set of nodes. More formally, the function $\Delta$ will specify the amount of initial mass starting at each node, and the function $T$ will give the sink capacity of each node, i.e. the most amount of mass we allow at a node after the spreading. We denote the {\em density} (of mass) at a node $v$ as the ratio of the amount of mass at $v$ over $\deg(v)$, and when we use density without any specific node, we mean the maximum density at any node. Naturally in a diffusion, we start with $\Delta$ having small support and high density, and the goal is to reach a state with bounded density enforced by the sink function. This gives a clean physical interpretation where paint (i.e. mass) spills from the source nodes and spreads over the graph and there is a sink at each node where up to a certain amount of paint can settle\footnote{There is also a ``paint spilling'' interpretation for personalized PageRank where instead of sinks holding paint, the paint dries (and settles) at a fixed rate when it pass through nodes. These are two very different mechanisms on how the mass settles.}.

In our work, we will always use the particular sink capacity function where $T(v)=\deg(v)$ for all nodes, i.e. density at most $1$. We extend the notation to write $\Delta(S)=\sum_{v\in S} \Delta(v)$ and $T(S)= \sum_{v\in S} T(v)$ for a subset of nodes $S$, and for our particular choice of sink capacity we have $T(S)=\vol(S)$. We also write $|\Delta|=\Delta(V)$ as the total amount of mass we start with, which remains constant throughout the diffusion as flow routing conserves mass.

Given initial mass $\Delta$ and routing $f$, the vector $\vec{m} = B^Tf+\Delta$ gives the amount of mass $m(v)$ at each node $v$ after the routing $f$, i.e. $m(v)=\Delta(v)+\sum_u f(u,v)$ is the sum of initial mass and the net amount of mass routed to $v$. We say $f$ is a {\em feasible routing} for a diffusion problem when $m(v) \le T(v)$ for all nodes, i.e. the mass obeys the sink capacity at each node. We say $v$'s sink is {\em saturated} if $m(v) \ge T(v)$ and $\ex(v)=\max\left(m(v)-T(v),0\right)$ the {\em excess} at $v$. Note there always exists some feasible routing as long as the total amount of mass $|\Delta|$ is at most $\vol(G)$, i.e. there is enough sink capacity over the entire graph to settle all the mass. This will be the case in the context of local clustering, and we will assume this implicitly through our discussion.

The goal of our diffusion problem is to find a feasible routing with low cost. We consider the $p$-norm of a routing $\norm{f}_p = (\sum_e f^p_e)^{1/p}$ as its cost. For example, when $p=2$ we can view the flow as electrical current, then the cost is the square root of the energy of the electrical flow; when $p=\infty$ the cost corresponds to the most congested edge's congestion of the routing. For $p<\infty$, the cost will naturally encourage the diffusion to send mass to saturate nearby sinks before expanding further, and thus our model inherently looks for local solutions.

For reader familiar with the network flow literature, in the canonical $p$-norm flow problem, we are given the exact amount of mass required at each sink, i.e. the inequality constraint is replaced by equality, and the high level question is {\em how} to route mass efficiently from given source(s) to destination(s). In our diffusion problem, as we have the freedom to choose the destination of mass as long as we obey the sink capacities, the essential question becomes {\em  where} to route the mass so the spreading can be low cost. Despite their similarity and close connection, we believe the distinct challenge of our $p$-norm diffusion problem poses a novel and meaningful direction to the classic problem of network flow.

\paragraph{The Dual Problem. } It is straightforward to check problem~\eqref{eq:dualLq} is the dual of $p$-norm flow diffusion, and strong duality holds for our problems so they have the same optimal value. For the dual problem, we view a solution $x$ as assigning heights to nodes, and the goal is to separate the nodes with source mass (i.e. seed nodes) from the rest of the graph. This is reflected in the objective where we gain by raising nodes with large source mass higher but loss by raising nodes in general. If we consider the height difference between an edge's two endpoints as the {\em length} of an edge, i.e. $|x(u)-x(v)|$, we constrain the separation of nodes with a budget for how much we can stretch the edges. More accurately, the $q$-norm of the vector of edge lengths (i.e. $\norm{Bx}_q$) is at most $1$. This naturally encourages stretching edges along a bottleneck (i.e. cut with small number of edges crossing it) around the seed nodes, since we can stretch each edge more when the number of edges is smaller (and thus raise seed nodes higher). The dual problem also inherently looks for local solutions as raising nodes far away from the source mass only hurts the objective. 

In contrast to random walk based linear operators such as personalized PageRank and heat kernel, even $2$-norm diffusion is non-linear due to the combinatorial constraint $B^Tf+\Delta \le T$. More generally, introducing non-linearity has proved to be very successful in machine learning, most notably in the context of deep neural networks. This may offer some intuition why $2$-norm diffusion achieves better results empirically comparing to personalized PageRank despite the connection between $2$-norm diffusion and spectral methods.\footnote{We note that algorithms (e.g.~\cite{ACL06}) based on random walks nonetheless introduce non-linearity (and also strong locality) to the underlying linear model in the form of approximation or regularization, whereas our model is intrinsically non-linear and strongly local.}




\section{Local Graph Clustering}
\label{sec:clustering}
In this section we discuss the optimal solutions of our diffusion problem (and its dual) in the context of local graph clustering. At a high level, we are given a set of seed nodes $S$ and want to find a cut of low conductance close to these nodes. 
Following prior work in local clustering, we formulate the goal as a promise problem, where we assume the existence of an unknown good cluster $C\subset V$ with $\vol(C)\leq \vol(V)/2$ and conductance $\phi_G(C) = \phi^*$. We consider a generic fixed $G=(V,E)$ and $p\in (1,\infty)$ through our discussion.

\subsection{Diffusion Setup}
To specify a particular diffusion problem and its dual, we need to provide the source mass $\Delta$, and recall we always set the sink capacity $T(v)=\deg(v)$. Given a set of seed nodes $S$, we pick a scalar $\delta$ and let
\begin{equation}
\label{eq:source}
\Delta(v) = \begin{cases} \delta \cdot \deg(v) & \mbox{ if } v\in S,\\
0 & \mbox{ otherwise}. \end{cases}
\end{equation}
Note this gives the total amount of mass $|\Delta|=\delta\cdot\vol(S)$. We will discuss the choice of $\delta$ shortly, but we start with a simple lemma on the locality of the optimal solutions for the primal and dual problems. 
\begin{lemma}
\label{lemma:optimalsupport}
Let $f^*$ and $x^*$ be optimal solutions of~\eqref{eq:primalLp} and~\eqref{eq:dualLq} respectively, $\supp(f^*)$ be the subset of edges with non-zero mass crossing them (i.e. the support of vector $f^*$), and $\supp(x^*)$ be the subset of nodes with strictly positive dual value. We have
\begin{enumerate}
    \item $|\supp(f^*)|\le \delta\cdot\vol(S)$,
	\item $\vol_G(\supp(x^*)) \le \delta\cdot\vol(S)$, and
	\item $x^*(u)>0$ only if $(B^T f^*+\Delta)(u)=\deg(u)$.
\end{enumerate}
\end{lemma}
\begin{proof}
For $p< \infty$, the optimal routing will never push mass out of a node $u$ unless $u$'s sink is saturated, i.e. $f^*(u,v)>0$ for $u,v$ only if $(B^Tf^*+\Delta)(u) = \deg(u)$. To see this, take the optimal primal solution $f^*$, and consider the decomposition of $f^*$ into flow paths, i.e., the path that the diffusion solution used to send each unit of mass from its source node to the sink at which it settled. If any node other than the last node on this path has remaining sink capacity, we can truncate the path at that node, and strictly reduce the total cost of the diffusion solution. As each unit of mass is associated with the sink of one node, the total amount of mass $\delta\cdot\vol(S)$ upper-bounds the total volume of the saturated nodes since it takes $\deg(v)$ amount of mass to saturate the node $v$. This observation proves the first claim. 

The dual variable $x^*(u)$ corresponds to the primal constraint $(B^T f^*+\Delta)(u)\le \deg(u)$, and it is easy the check the third claim of the lemma is just complementary slackness. The second claim follows from the first and the third claim.
\end{proof}

Now we discuss how to set $\delta$. The intuition is that we want the total amount of source mass starting in $C$ to be a constant factor larger than the volume of $C$, say $\Delta(C)\ge 2\vol(C)$ (any constant reasonably larger than $1$ would work). The reason is that in such scenario, at least $\Delta(C) - \vol(C) \ge \vol(C)$ amount of mass has to be routed out of $C$ since the nodes in $C$ have total sink capacity $\vol(C)$. When $C$ is a cut of low conductance, any feasible routing must incur a large cost since $\vol(C)$ amount of mass has to get out of $C$ using a relatively small number of discharging edges. In this case, the optimal dual solution $x^*$ will certify the high cost of any feasible primal solution. Naturally, the appropriate value of $\delta$ to get $\Delta(C)\ge 2\vol(C)$ depends on how well the seed set $S$ overlaps with $C$. Suppose $\vol(S\cap C)\geq \alpha\vol(C)$, then we can set $\delta = 2/\alpha$. Without loss of generality, we assume the right value of $\delta$ is known since otherwise we can employ binary search to find a good value of $\delta$. 

More formally, we derive a low conductance cut from $x^*$ using the standard sweep cut procedure. In our case, because $x^*$ has bounded support, the procedure can be implemented in $O(\delta\cdot\vol(S))$ total work.
\begin{figure}[h]
\fbox{\parbox{\textwidth}{
\begin{enumerate}
\item Sort the nodes in decreasing order by their values in $x^*$.
\item For each $i\geq 1$ such that the $i$-th node still has strictly positive dual value, consider the cut containing the first $i$ nodes. Among all these cuts (also called {\em level cuts}) output the one with the smallest conductance. 
\end{enumerate}
}}
\caption{The Sweep Cut Procedure.}
\label{fig:sweepcut}
\end{figure}

\subsection{Local Clustering Guarantee}
\begin{theorem}
\label{thm:main}
Given a set of seed nodes $S$, suppose there exists a cluster $C$ such that
\begin{enumerate}
    \item $\vol(S\cap C)\geq \alpha\vol(C)$ for some $\alpha\in (0,1]$,
    \item $\vol(S\cap C)\geq \beta\vol(S)$ for some $\beta\in (0,1]$.
\end{enumerate}
Then the cut $\tilde{C}$ returned by the sweep cut procedure on the optimal dual solution $x^*$ satisfies
\[
\phi(\tilde{C})\leq O\left(\frac{\phi(C)^{1/q}}{\alpha\beta}\right)
\]
where $q\in (1,\infty)$ is the norm used in~\eqref{eq:dualLq}
\end{theorem}
Note the sweep cut computation only requires the dual solution $x^*$, while the primal solution $f^*$ and the values of $\alpha,\beta$ are only for analysis. Recall we want to set $\delta = 2/\alpha$ in~\eqref{eq:source} to formulate the dual problem, but we assume $\delta$ is known via binary search. We also assume the entire graph is larger than the total amount of source mass so the primal is feasible and the dual is bounded. As summarized below,
\begin{assumption}
\label{assum:source}
The source mass function $\Delta$ in our problem formulation as specified in~\eqref{eq:source} satisfies $\delta = 2/\alpha$, which gives $\Delta(C) \geq 2\vol(C)$ and $|\Delta|=\Delta(S) \le 2\vol(C)/\beta< \vol(G)$.
\end{assumption}
Not surprisingly, our theorem is more meaningful when the given seed set $S$ has a good overlap with some low conductance cut $C$, i.e. when $\alpha,\beta$ are bounded away from $0$. In particular, suppose $\alpha,\beta$ are both at least $\frac{1}{\log^t\left(\vol(C)\right)}$ for some constant $t$, then the bound in our theorem becomes $\Otil(\phi(C)^{1/q})$, where we follow the tradition of using $\Otil$ to hide poly-logarithmic factors. In particular, for $2$-norm diffusion (i.e. $p=q=2$) this matches the bound achieved by spectral and random walk based methods in this setting, and for $p$-norm diffusion with $p$ approaches $\infty$ (i.e. $q$ tends to $1$), this matches the guarantees of previous flow based methods in this setting, e.g.~\cite{WFHM2017,OZ14}. We prove Theorem~\ref{thm:main} in the rest of the section.

We start with the simple lemma stating that the objective value of the optimal dual (and primal) solution must be large.
\begin{lemma}
\label{lem:top}
Suppose $k=|E(C,V\setminus C)|$ is the cut-size of $C$, then 
\[
(\Delta - \ddeg)^Tx^* = \norm{f^*}_p \ge \vol(C)/k^{1/q}.
\]
\end{lemma}
\begin{proof}
 By Assumption~\ref{assum:source} there is least $\Delta(C)\ge 2\vol(C)$ amount of source mass trapped in $C$ at the beginning. Since the sinks of nodes in $C$ can settle $\vol(C)$ amount of mass, the remaining at least $\vol(C)$ amount of excess needs to get out of $C$ using the $k$ cut edges. We focus on the cost of $f^*$ restricted to these edges alone. Since $p>1$, the cost is the smallest if we distribute the routing load evenly on the $k$ edges, and it is simple to see this incurs cost 
\[
\left(k\cdot \bigg(\frac{\vol(C)}{k}\bigg)^p\right)^{1/p}= \vol(C)/k^{(p-1)/p}.
\]
The total cost of $f^*$ must be at least the cost incurred just by routing the excess out of $C$.
\end{proof}

Recall that we define the length of an edge $e=(u,v)$ to be $l(e)=|x^*(u)-x^*(v)|$. The actual dual solution may incur edges with tiny non-zero length which causes difficulties in the analysis. Thus, we define the following perturbed edge length so that any non-zero edge length is at least $1/\vol(C)^{1/q}$. Note this is only for analysis purpose and doesn't require changing $x^*$.
\begin{equation}
\label{eq:length}
\tilde{l}(e) = \begin{cases} \max\left(\frac{1}{\vol(C)^{1/q}},l(e)\right) & \mbox{ if } l(e)>0,\\
0 & \mbox{ otherwise}. \end{cases}
\end{equation}
Note the constraint in the dual problem gives
\begin{align*}
\sum_{e=(u,v)}|x^*(u)-x^*(v)|\cdot l(e)^{q-1} \ = \ \sum_{e=(u,v)}|x^*(u)-x^*(v)|^{q}=\norm{Bx^*}^q_q\le 1.
\end{align*}
The next lemma states that the perturbations on edges lengths are small enough so the above quantity remains small.
\begin{lemma}
\label{lem:bottom}
$\sum_{e=(u,v)}|x^*(u)-x^*(v)|\cdot \tilde{l}(e)^{q-1} \le 1+ \frac{2}{\beta}$.
\end{lemma}
\begin{proof}
 \begin{equation*}
 \sum_{e=(u,v)}|x^*(u)-x^*(v)|\cdot \tilde{l}(e)^{q-1}
 \le \sum_{e=(u,v)} \tilde{l}(e)^{q}
 = \sum_{e:l(e) = \tilde{l}(e)} l(e)^{q}
 +\sum_{e:l(e) < \tilde{l}(e)} \frac{1}{\vol(C)} \le 1 +  \frac{2}{\beta}.
 \end{equation*}
 The second to last equality follows from the fact that our perturbation only increases the lengths to $\frac{1}{\vol(C)^{1/q}}$. The last inequality follows from that we only increase the length of an edge when its original edge is positive, which means at at least one of its endpoints has positive dual variable value. From Assumption~\ref{assum:source} that $\delta=2/\alpha$ we know that the total amount of mass is at most $\frac{2}{\alpha}\vol(S)$. Together with the conditions in Theorem~\ref{thm:main} we get $\frac{2}{\alpha}\vol(S)\leq \frac{2}{\beta}\vol(C)$. This upper-bounds $\vol_G(\supp(x^*))$ by Lemma~\ref{lemma:optimalsupport}. Thus, the number of edges with positive $l(e)$ is also at most $\frac{2}{\beta}\vol(C)$.
\end{proof}

Consider the sweep cut procedure where we order the nodes by their dual values in $x^*$, and for any $h> 0$ denote the cut $S_h = \left\{u | x^*(u) \ge h\right\}$ to be the set of nodes with dual value larger than $h$. We only need to consider $S_h$ when $h$ equals to the strictly positive dual value of some node in the support of $x^*$, and the sweep cut procedure will examine all such cuts. We proceed to argue that among these level cuts, there must exists some $h$ where $\phi(S_h)$ satisfies the bound in Theorem~\ref{thm:main}, and thus prove the main theorem. 

We start with rewriting the dual objective and constraint using the level cuts.
\begin{claim} 
\label{claim:integral}
With level cuts $S_h$ as defined above, we have
\[
\int_{h=0}^{\infty} \left(\Delta(S_h)-\vol(S_h)\right)dh \ge \vol(C)/k^{1/q},
\]
and
\[
\int_{h=0}^{\infty} \sum_{e\in E(S_h,V\setminus S_h)} \tilde{l}(e)^{q-1}dh \le 1+ \frac{2}{\beta}. 
\]
\end{claim}
\begin{proof} 
Both claims follow from changing the order of summation to get
\[
(\Delta-\ddeg)^Tx^* = \int_{h=0}^{\infty} \left(\Delta(S_h)-\vol(S_h)\right)dh,
\]
and 
\begin{align*}
\sum_{e=(u,v)}|x^*(u)-x^*(v)|\cdot \tilde{l}(e)^{q-1}
=   \int_{h=0}^{\infty} \sum_{e\in E(S_h,V\setminus S_h)} \tilde{l}(e)^{q-1}dh,    
\end{align*}
and then invoke Lemma~\ref{lem:top} and Lemma~\ref{lem:bottom} respectively. To see first claim, pick any node $v$ in the first equation. $v$ contributes $(\Delta(v)-\deg(v))\cdot x^*(v)$ to the left hand side, and the same amount to the right hand side as $v$ is in the level cuts for all $h\in(0,x^*(v)]$.

For the second claim, pick any edge $e=(u,v)$, the edge will cross the level cuts $E\left(S_h,V\setminus S_h\right)$ for all $h\in (x^*(v),x^*(u)]$ (assuming wlog $x^*(u)\ge x^*(v)$), so the contribution from any edge will be the same to both sides of the equation.
\end{proof}

Using Claim~\ref{claim:integral}, we take the ratio to get
\[
\frac{\int_{h=0}^{\infty} \sum_{e\in E(S_h,V\setminus S_h)} \tilde{l}(e)^{q-1}dh}{\int_{h=0}^{\infty} \left(\Delta(S_h)-\vol(S_h)\right)dh} \leq \frac{3 k^{1/q}}{\beta\vol(C)},
\]
which means there must exist some $h$ with $S_h$ non-empty and
\begin{equation}
\label{eq:first}
\frac{\sum_{e\in E(S_h,V\setminus S_h)} \tilde{l}(e)^{q-1}}{\Delta(S_h)-\vol(S_h)} \le \frac{3 k^{1/q}}{\beta\vol(C)}.
\end{equation}
All it remains is to connect the left hand side in the above inequality to the conductance of $S_h$. For the denominator, since the source mass has density at most $\delta=2/\alpha$ at any node, we get 
\begin{equation}
\label{eq:second}
\Delta(S_h)-\vol(S_h)\le \frac{2}{\alpha}\vol(S_h).
\end{equation}
For the numerator, any edge $e=(u,v)$ crossing a level cut $S_h$ must have dual values $x^*(u),x^*(v)$ on different sides of $h$, thus having non-zero length $l(e)$, which means $\tilde{l}(e)$ is at least $1/\vol(C)^{1/q}$. This gives
\begin{equation}
\label{eq:third}
\sum_{e\in E(S_h,V\setminus S_h)} \tilde{l}(e)^{q-1}\ge \frac{|E(S_h,V\setminus S_h)|}{\vol(C)^{(q-1)/q}}.
\end{equation}
Put~\eqref{eq:first},~\eqref{eq:second} and~\eqref{eq:third} together we get
\[
\phi(S_h) = \frac{|E(S_h,V\setminus S_h)|}{\vol(S_h)}\le \frac{6k^{1/q}}{\alpha\beta\vol(C)^{1/q}}=\frac{6\phi(C)^{1/q}}{\alpha\beta},
\]
which proves our main theorem.

\section{Strongly Local Algorithm}
\label{sec:algorithm}
For general constrained convex optimization problem, the state of the art solvers are by interior point techniques. However, these methods start with a fully dense initial point and then iteratively solve a linear system to obtain the update direction, which quickly become prohibitive when the input size is only moderately large. For example, the iterates of CVX~\cite{cvx} fail to converge when solving the dual problem~\eqref{eq:dualLq} on a $60 \times 60$ grid graph with initial mass $|\Delta|=12000$ and $p=4$.

As noted in Lemma~\ref{lemma:optimalsupport}, the support of optimal primal and dual solutions are bounded by the total amount of initial mass $|\Delta|$. 
Hence, we exploit locality and propose a randomized coordinate descent variant that solves an equivalent regularized formulation of~\eqref{eq:dualLq}:
\begin{equation}
	\min_{x\ge0}~F(x) := \tfrac{1}{q}\|Bx\|_q^q - x^T(\Delta - \ddeg).
	\label{eq:dualLqreg}
\end{equation}
Our algorithm is strongly local, that is, its running time depends on $|\Delta|$ rather than the size of the whole graph.

\begin{proposition}[Equivalence]
	\label{lemma:equivalence}
	The solution sets of problems \eqref{eq:dualLq} and \eqref{eq:dualLqreg} are scaled versions of each other.
\end{proposition}

\begin{proof}
First note that constant non-zero vectors cannot be solutions to any of the problems \eqref{eq:dualLq} and \eqref{eq:dualLqreg}. This is because we pick $\Delta$ such that $ \sum_{i} \Delta(i) \le \sum_{i} \deg(i)$ and an all-zero vector is feasible and has better objective function value than any constant non-zero vector.
Let $x^*$ denote a non-constant solution of \eqref{eq:dualLqreg}. Then $x^*$ satisfies the optimality conditions of \eqref{eq:dualLqreg}
\begin{itemize}
    \item $-(\Delta - \ddeg) - y + \frac{1}{q}\nabla \|Bx\|_q^q=0$
    \item $y(i) x(i) = 0 \ \forall i\in V$
    \item $x,y\ge 0$,
\end{itemize}
for some optimal dual variables $y^*$.

The optimality conditions of problem \eqref{eq:dualLq} are 
\begin{itemize}
    \item $-(\Delta - \ddeg) - y + \lambda \nabla \|Bx\|_q^q=0$
    \item $y(i) x(i) = 0 \ \forall i\in V$
    \item $\lambda (1 - \|Bx\|_q^q) = 0$
    \item $\norm{Bx}_q^q \le 1$
    \item $x,y\ge 0$, $\lambda \ge 0$,
\end{itemize}
where $y$ are the dual variables for the constraint $x\ge 0$, and $\lambda$ is the dual variable for $\norm{Bx}_q^q \le 1$. Observe that by setting $x:=x^*/\|Bx^*\|_q^q$, $\lambda := \tfrac{1}{q}(\|Bx^*\|_q^q)^{q-1}$ and $y:=y^*$, then the triplet $x,\lambda,y$ satisfies the latter optimality conditions.

Let us now prove the reverse. Note that because we pick $\Delta$ such that there exists at least one negative component in $-(\Delta - \ddeg)$, then $\lambda=0$ can never be an optimal dual variable. Thus $\lambda>0$.
Let $\hat{x},\hat{\lambda},\hat{y}$ be a solution to the latter optimality conditions. Observe that $x:=(q \hat{\lambda})^{\frac{1}{q-1}}\hat{x}$ and $y:=\hat{y}$ satisfy the former optimality conditions. This means that the solution sets of the two problems are scaled versions of each other.
\end{proof}

Proposition~\ref{lemma:equivalence} is important because the output of the Sweep Cut procedure is the same for both solutions. This is because the output of Sweep Cut depends only on the ordering of dual variables and not on their magnitude. Therefore, we can use the solution of any of these problems for the local graph clustering problem.

Following Proposition~\ref{lemma:equivalence}, it is easy to see that if one of \eqref{eq:dualLq} or \eqref{eq:dualLqreg} is unbounded, the other must also be unbounded. Moreover, one can easily verify that the $q$-norm regularized problem~\eqref{eq:dualLqreg} is the dual of an equivalent $p$-norm flow diffusion problem
\begin{equation}
	\label{eq:equivprimalLp}
	\begin{split}
		\min & \ \tfrac{1}{p}\|f\|_p^p \\
		\mbox{s.t.}~ & \ B^T f + \Delta \le \ddeg
	\end{split}
\end{equation}
and that strong duality holds, because any $x>0$ is a Slater point for~\eqref{eq:dualLqreg}. Throughout the discussions in this section, we assume that Assumption~\ref{assum:source} holds, i.e. $|\Delta| \le \vol(G)$, so \eqref{eq:equivprimalLp} is feasible, and hence \eqref{eq:dualLqreg} is bounded. Although the two formulations \eqref{eq:dualLq} and \eqref{eq:dualLqreg} are inherently equivalent, the computational advantage of \eqref{eq:dualLqreg} permits the use of many off the shelf first order optimization methods, which is crucial for obtaining strongly local algorithms.

Because the function $F(x)$ in the objective of \eqref{eq:dualLqreg} has non-Lipschitz gradient for any $q < 2$, directly minimizing $F(x)$ requires step-sizes that go to zero to guarantee convergence, which lead to slow practical and worst-case convergence rate. To cope with this, we smooth $F(x)$ by perturbing the $q$-norm term around zero and consider the following {\em globally} smoothed problem
\begin{equation}
	\min_{x\ge0} \ F_\mu(x) :=\frac{1}{q}\sum_{(i,j) \in E} ((x(i)-x(j))^2+\mu^2)^{q/2} - x^T(\Delta-\ddeg),
	\label{eq:globalLq}
\end{equation}
where $\mu > 0$ is a smoothing parameter. The proposed numerical scheme in Algorithm~\ref{alg:coordinatedescent} solves the smoothed problem~\eqref{eq:globalLq}. In fact, we will prove in Theorem~\ref{thm:iterationcomplexity} and Corollary~\ref{cor:runtime} that, by setting out $\mu$ appropriately, Algorithm~\ref{alg:coordinatedescent} finds $\epsilon$ accurate solution to the original regularized dual problem~\eqref{eq:dualLqreg} in strongly local running time.

{\centering
\begin{minipage}{.7\linewidth}
\begin{algorithm}[H]
	\caption{Coordinate solver for smoothed dual problem}
	\quad {\bf Initialize}: $x_0 = 0$ 
	
	\quad {\bf For} $k = 0, 1,2,\ldots,$ {\bf do}
	
	\qquad Set $S_k = \{ i \in V ~|~ \nabla_i F_\mu(x_k) < 0\}$.
	
	\qquad Pick $i_k \in S_k$ uniformly at random.
	
	\qquad Update $x_{k+1} = x_k - \dfrac{\mu^{2-q}}{\deg(i_k)}\nabla_{i_k}F_\mu(x_k) e_{i_k}$.
	
	\qquad {\bf If} $S_k = \emptyset$ {\bf then return} $x_k$.
	\label{alg:coordinatedescent}
\end{algorithm}
\end{minipage}
\par}
\vspace{5mm}

In the context of $p$-norm flow diffusion, coordinate method enjoys a natural combinatorial interpretation: each coordinate update corresponds to sending mass from a node to its neighboring nodes along incident edges. We now formally explain this combinatorial interpretation as it will help us develop a physical intuition on the numerical steps of Algorithm~\ref{alg:coordinatedescent} in terms of diffusing mass in the underlying graph. The basic setting is as follows. Each node $v$ maintains a height $x(v)$. For any fixed $q\in(1,2]$, the amount of flow from node $u$ to node $v$ is determined\footnote{Note that the computation in~\eqref{eq:flowvalue} of flow values from dual variables $x$ comes directly from the primal-dual stationary condition of \eqref{eq:globalLq}.} by the relative heights $x(u)$ and $x(v)$ and the smoothing parameter $\mu \ge0$,
\begin{equation}
\label{eq:flowvalue}
	\flow_q(u, v; x, \mu) := ((x(u)-x(v))^2 + \mu^2)^{q/2-1}(x(u)-x(v)).
\end{equation}
Therefore, given $x$ and $\mu$, the corresponding mass and the excess at node $v$ are
\[
	\begin{split}
		m(v; x, \mu) &= \Delta(v) + \sum_{u \sim v} \flow_q(u, v; x, \mu), \\
		\ex(v; x, \mu) &= \max\{0, m(v; x, \mu) - \deg(v)\}.
	\end{split}
\]
One can easily verify that, under these definitions,
\[
	- \nabla_i F_\mu(x) = m(i; x, \mu) - \deg(i), \quad\forall~i \in V,
\]
and that the steps we layout in Figure~\ref{fig:comb_alg} are indeed equivalent to the steps in Algorithm~\ref{alg:coordinatedescent}.

\begin{figure}[h]
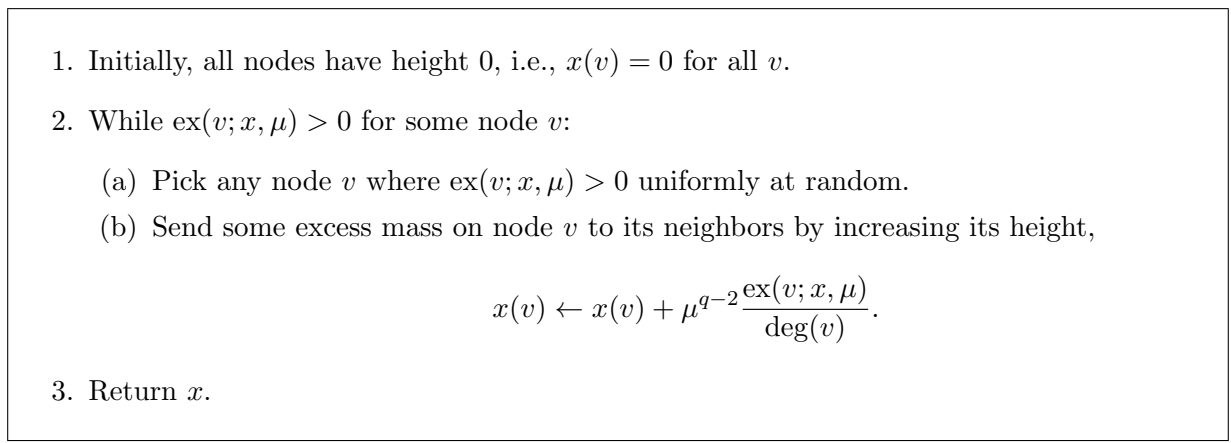

\centering
\fbox{\parbox{\textwidth}{
	\begin{enumerate}
		\item Initially, all nodes have height 0, i.e., $x(v) = 0$ for all $v$.
		\item While $\ex(v; x, \mu) > 0$ for some node $v$:
			\begin{enumerate}[label=(\alph*)]
				\item Pick any node $v$ where $\ex(v; x, \mu) > 0$ uniformly at random.	
				\item Send some excess mass on node $v$ to its neighbors by increasing its height,
					\[
						x(v) \gets x(v) + \mu^{q-2}\frac{\ex(v; x, \mu)}{\deg(v)}.
					\]
			\end{enumerate}
		\item Return $x$.
	\end{enumerate}
}}
\caption{A combinatorial description of Algorithm~\ref{alg:coordinatedescent}}
\label{fig:comb_alg}
\end{figure}

The rest of this section is devoted to proving that Algorithm~\ref{alg:coordinatedescent} solves \eqref{eq:dualLqreg} in strongly local running time. We start by introducing a related problem that will help our analysis.

Let $x_\mu^*$ denote an optimal solution of~\eqref{eq:globalLq}. We define a corresponding {\em locally} smoothed problem
\begin{equation}
\label{eq:localLq}
	\min_{x\ge0} \ F_\mu^l(x),
\end{equation}
where $F_\mu^l(x)$ is obtained by perturbing the $q$-norm term around zero on the edges defined by $\supp(Bx_\mu^*) \subseteq E$,
\[
	F_\mu^l(x) ~:=~ 
	\frac{1}{q}\sum_{(i,j) \in \supp(Bx_\mu^*)} ((x(i)-x(j))^2 + \mu^2)^{q/2} ~+~ 
	\frac{1}{q}\sum_{(i,j) \not\in \supp(Bx_\mu^*)} |x(i)-x(j)|^q ~~-~~ x^T(\Delta-\ddeg).
\]
The following two lemmas, Lemma~\ref{lemma:locality_optsol} and Lemma~\ref{lemma:smoothapprox}, show that the smoothed objective value $F_\mu^l(x)$ is ``absolute locally'' close to the original objective value $F(x)$. That is, for any $x$, the maximum difference in the objective values between $F_\mu^l(x)$ and $F(x)$ depends only on $|\Delta|$ and not the dimension of the ambient space.

\begin{lemma}[Locality]
\label{lemma:locality_optsol}
	The number of edges defined by $\supp(Bx_\mu^*)$ is bounded by the total amount of initial mass, i.e., 
	\[
		|\supp(Bx_\mu^*)| < |\Delta|.
	\]
\end{lemma}

\begin{proof}
	By the first-order optimality condition of~\eqref{eq:globalLq}, for all $i \in \supp(x_\mu^*)$,
	\[
	 	\nabla_i F_\mu(x_\mu^*) = \sum_{j\sim i} ( (x_\mu^*(i) - x_\mu^*(j))^2 + \mu^2 )^{q/2-1} (x_\mu^*(i) - x_\mu^*(j)) - \Delta_i + \deg(i)=  0.
	\]
	Hence,
	\[
	\begin{aligned}
		|\supp(Bx_\mu^*)| &\le \vol_G(\supp(x_\mu^*)) \\
		&= \sum_{i \in \supp(x_\mu^*)} \deg(i) \\
		&= \sum_{i \in \supp(x_\mu^*)} \Delta(i) \ - \sum_{i \in \supp(x_\mu^*)}\sum_{j\sim i} ( (x_\mu^*(i) - x_\mu^*(j))^2 + \mu^2 )^{q/2-1} (x_\mu^*(i) - x_\mu^*(j))\\
		&= \sum_{i \in \supp(x_\mu^*)} \Delta(i) \
			- \underbrace{\sum_{\substack{i,j \in \supp(x_\mu^*) \\ i \sim j}}( (x_\mu^*(i) - x_\mu^*(j))^2 + \mu^2 )^{q/2-1} (x_\mu^*(i) - x_\mu^*(j))}_{= 0} \\
		&      \hspace{28mm} - \underbrace{\sum_{\substack{i \in \supp(x_\mu^*) \\ j \not\in \supp(x_\mu^*) \\ i \sim j}}( (x_\mu^*(i) - x_\mu^*(j))^2 + \mu^2 )^{q/2-1} (x_\mu^*(i) - x_\mu^*(j))}_{>0} \\
		&\le |\Delta|.
	\end{aligned}
	\]
\end{proof}

\begin{lemma}[Local smooth approximation]
\label{lemma:smoothapprox}
	For any $x$, we have that
	\begin{equation}
	\label{eq:smoothapprox}
		F(x) \le F_\mu^l(x) \le F(x) + \tfrac{1}{q}\mu^q|\Delta|.
	\end{equation}
\end{lemma}

\begin{proof}
	Denote
	\[
		\sigma_{\mu,q}(x) ~~:= \sum_{(i,j) \in \supp(Bx_\mu^*)} ((x(i)-x(j))^2 + \mu^2)^{q/2} ~~+ \sum_{(i,j) \not\in \supp(Bx_\mu^*)} |x(i)-x(j)|^q.
	\]
	It suffices to show
	\[
		\|Bx\|_q^q ~\le~ \sigma_{\mu,q}(x) ~\le~ \|Bx\|_q^q + \mu^q |\Delta|.
	\]
	First of all, it is straightforward to see that 
	\[
		\|Bx\|_q^q
		~\le \sum_{(i,j) \in \supp(Bx_\mu^*)} ((x(i) - x(j))^2 + \mu^2)^{q/2} \ + \sum_{(i,j) \not\in \supp(Bx_\mu^*)} ((x(i) - x(j))^2)^{q/2}
		~=~ \sigma_{\mu,q}(x).
	\]
	On the other hand, since $q\in(1,2]$, the function $h : \mathbb{R}_{\ge0} \rightarrow \mathbb{R}_{\ge0}$ 
	given by $h(s) = s^{q/2}$ is concave and $h(0) = 0$.
	So $h$ is sub-additive. Thus
	\[
	\begin{split} 
		\sigma_{\mu,q}(x)
		~\le&\sum_{(i,j) \in E} ((x(i) - x(j))^2)^{q/2} +  \sum_{(i,j) \in \supp(Bx_\mu^*)} \mu^q \\
		=&~\|Bx\|_q^q + \mu^q |\supp(Bx_\mu^*)| \\
		\le&~\|Bx\|_q^q + \mu^q |\Delta|,
	\end{split}
	\]
	where the last inequality is due to Lemma~\ref{lemma:locality_optsol}.
\end{proof}

The approximation bound in Lemma~\ref{lemma:smoothapprox} means that, by setting out $\mu$ inversely proportional to $|\Delta|$, a minimizer of $F_\mu^l(x)$ can be easily turned into an approximate minimizer of $F(x)$, up to any desired $\epsilon$ accuracy. Therefore, in order to obtain an $\epsilon$ accurate solution to the regularized dual problem~\eqref{eq:dualLqreg}, it suffices to minimize $F_\mu^l(x)$. In Theorem~\ref{thm:convergence} we show that, solving the globally smoothed problem~\eqref{eq:globalLq} effectively minimizes $F_\mu^l(x)$.

In the following we analyze important properties of $F_\mu(x)$ which also extend to $F_\mu^l(x)$ on $\supp(x_\mu^*)$. Lemmas~\ref{lemma:lipschitz},~\ref{lemma:monotonegradient} and~\ref{lemma:boundedgradients} provide some technical results which will be used to prove Theorem~\ref{thm:convergence} on the convergence of Algorithm~\ref{alg:coordinatedescent}.

\begin{lemma}[Lipschitz continuity]
\label{lemma:lipschitz}
	$\nabla F_\mu(x)$ is Lipschitz continuous with coordinate Lipschitz constant $L_i = \deg(i)\mu^{q-2}$.
\end{lemma}

\begin{proof}
	The first and second order derivatives of $F_\mu(x)$ are
	\[
	\begin{split}
		\nabla F_\mu(x) &= B^T C(x) B x - \Delta + \ddeg, \\
		\nabla^2 F_\mu(x) &= B^T \tilde{C}(x) B,
	\end{split}
	\]
	where both $C(x)$ and $\tilde{C}(x)$ are diagonal matrices of size $|E|\times|E|$ whose diagonal entries correspond to edges $(i,j) \in E$,
	\[
	\begin{split}
		[C(x)] _{(i,j),(i,j)}&= \left( (x(i) - x(j))^2 + \mu^2 \right)^{\frac{q}{2}-1}, \\
		[\tilde{C}(x)]_{(i,j),(i,j)} &= \left( (x(j)-x(j))^2 + \mu^2 \right)^{\frac{q}{2}-2} ((q-1)(x(i) - x(j))^2 + \mu^2 ).
	\end{split}
	\]
	In order to obtain coordinate Lipschitz constant $L_i$ of $\nabla_i F_\mu(x)$, we upper bound $[\tilde{C}(x)]_{(i,j),(i,j)}$ by
	\begin{align*}
		[\tilde{C}(x)]_{(i,j),(i,j)} 
		& = \left( (x(i)-x(j))^2 + \mu^2 \right)^{\frac{q}{2}-2} ((q-1)(x(i) - x(j))^2 + \mu^2 )\\
		& \le  \left( (x(i)-x(j))^2 + \mu^2 \right)^{\frac{q}{2}-1} \frac{(q-1)(x(i) - x(j))^2 + \mu^2}{(x(i)-x(j))^2 + \mu^2}.
	\end{align*}
	Since $q\in(1,2]$, we get that
	\[
		\frac{(q-1)(x(i) - x(j))^2 + \mu^2}{(x(i)-x(j))^2 + \mu^2 } \le 1,
	\]
	and thus
	\[
		[\tilde{C}(x)]_{(i,j),(i,j)} \le \left( (x(i)-x(j))^2 + \mu^2 \right)^{\frac{q}{2}-1}  \le \mu^{q-2}.
	\]
	This means that 
	\[
		\nabla^2 F_\mu(x) = B^T \tilde{C}(x) B \preceq  \mu^{q-2} B^TB.
	\]
	Therefore
	\[
		L_i = \mu^{q-2} e_i^TB^TBe_i = \deg(i) \mu^{q-2}.
	\]
\end{proof}

\begin{lemma}[Monotone gradients]
\label{lemma:monotonegradient}
	For any $x$, the following hold
	\begin{itemize}
		\item $h_{ii}(t) := \nabla_i F_\mu (x + t e_i)$ is strictly monotonically increasing in $t$,
		\item $h_{ij}(t) := \nabla_i F_\mu (x + t e_j)$ is strictly monotonically decreasing in $t$ if $j\sim i$ and constant if $j\not\sim i$.
	\end{itemize}
\end{lemma}

\begin{proof}
	Expand $\nabla_i F_\mu(x)$ as
	\[
		\nabla_i F_\mu(x) = \sum_{j\sim i} \left( (x(i) - x(j))^2 + \mu^2 \right)^{q/2-1} (x(i) - x(j)) - \Delta(i) + \deg(i).
	\]
	So
	\[
		h_{ii}(t) = \sum_{j\sim i} \left( (x(i) + t - x(j))^2 + \mu^2 \right)^{q/2-1} (x(i) + t- x(j)) - \Delta(i) + \deg(i).
	\]
	Each term in the above sum is a function $g(y) := \left(y^2 + \mu^2 \right)^{q/2-1}y$. The derivative of $g(y)$ is 
	\[
		g'(y) = (y^2 + \mu^2)^{q/2-2} ((q-1) x^2 + \mu^2) > 0,
	\]
 	therefore, $g(y)$ is strictly monotonically increasing, and hence so it $h_{ii}(t)$.
	Similarly, $h_{ij}(t)$ is equal to 
	\[
		h_{ij}(t) = \sum_{j\sim i} \left( (x(i) - t - x(j))^2 + \mu^2 \right)^{q/2-1} (x(i) - t- x(j)) - \Delta(i) + \deg(i).
	\]
	Using the same reasoning as above we get that function $h_{ij}(t)$ is strictly monotonically decreasing if $j \sim i$ and is constant if $j \neq i$ and $j \not\sim i$.
\end{proof}

\begin{lemma}[Non-positive gradients]
\label{lemma:boundedgradients}
	For any iteration $k$ and node $i$ in Algorithm~\ref{alg:coordinatedescent}, $\nabla_i F_\mu(x_k) \le 0$, $\forall i \in \supp(x_k)$.
\end{lemma}

\begin{proof}
	Suppose for some iteration $k \ge 0$ we have that $\nabla_i F_\mu(x_k) \le 0$ for every $i \in \supp(x_k)$, and that $S_k \neq \emptyset$. 
	Say we have picked to update some $i_k \in S_k$. 
	This means $\nabla_{i_k} F_\mu(x_k) < 0$. 
	It follows from Lipschitz continuity (cf. Lemma~\ref{lemma:monotonegradient}) that
	\[
		\begin{split}
			|\nabla_{i_k} F_\mu(x_{k+1}) - \nabla_{i_k} F_\mu(x_k)|
			&=~\left|\nabla_{i_k} F_\mu(x_k - \tfrac{\mu^{2-q}}{\deg(i_k)}\nabla_{i_k}F_\mu(x_k)e_{i_k}) - \nabla_{i_k} F_\mu(x_k)\right| \\
			&\le~L_{i_k}\left|\tfrac{\mu^{2-q}}{\deg(i_k)}\nabla_{i_k}F_\mu(x_k)\right|
			=~|\nabla_{i_k} F_\mu(x_k)|.
		\end{split}
	\]
	Since $\nabla_{i_k} F_\mu(x_k) < 0$, we must have $\nabla_{i_k} F_\mu(x_{k+1}) \le 0$. Moreover, since
	\[
		x_{k+1}(i_k) = x_k(i_k) - \tfrac{\mu^{2-q}}{\deg(i_k)}\nabla_{i_k} F_\mu(x_k) > x_k(i_k),
	\]
	and $x_{k+1}(j) = x_k(j)$ for all $j \neq i_k$, by Lemma~\ref{lemma:monotonegradient} we know that 
	$\nabla_j F_\mu(x_{k+1}) < \nabla_j F_\mu(x_k) \le 0$ for every $j \sim i_k$, 
	and $\nabla_j F_\mu(x_{k+1}) = \nabla_j F_\mu(x_k)$ for every $j \not\sim i_k$.
	The proof is complete by noting that $\nabla F_\mu(x_0) \le 0$ holds trivially.
\end{proof}

Note that $F_\mu(x)$ is not strongly convex in general, but locality gives us strong convexity, as shown in the following lemma.

\begin{lemma}[Strong convexity]
	\label{lemma:strongconvexity}
	Let $x_\mu^* \in \argmin_{x\ge0} F_\mu(x)$. We have that
	\begin{equation}
		F_\mu(y) \ge F_\mu(x) + \nabla F_\mu(x)^T(y-x) + \frac{\gamma}{2}\|y-x\|_2^2, \quad \forall~x,y \in U(x_\mu^*),
	\end{equation}
	where
	\[
		U(x_\mu^*) := \left\{x \in \mathbb{R}^{|V|} ~\big|~ \supp(x) \subseteq \supp(x_\mu^*) 
	    				~\mathrm{and}~ (x(i)-x(j))^2 + \mu^2 \le |\Delta|^{2p-2}~\forall(i,j) \in E\right\},
	\]
	and the strong convexity parameter $\gamma$ satisfies
	\[
		\gamma > \frac{1}{(p-1)|\Delta|^p}.
	\]
\end{lemma}

\begin{proof}	
	Recall the second order derivative of $F_\mu(x)$ is
	\[
		\nabla^2 F_\mu(x) = B^T \tilde{C}(x) B,
	\]
	where $\tilde{C}(x)$ is a diagonal matrix of size $|E|\times|E|$ whose diagonal entries correspond to edges $(i,j) \in E$,
	\[
	\begin{split}
		[\tilde{C}(x)]_{(i,j),(i,j)} 
		&= \left( (x(i)-x(j))^2 + \mu^2 \right)^{\frac{q}{2}-2} ((q-1)(x(i) - x(j))^2 + \mu^2 )\\
		&\ge (q-1)\left( (x(i)-x(j))^2 + \mu^2 \right)^{\frac{q}{2}-1}.
	\end{split}
	\]
	Let $S:=\supp(x_\mu^*)$.
	Let $B_S$ denote the sub-matrix of $B$ with columns chosen such that they correspond to nodes in $\supp(x_\mu^*)$.
	In order to lower bound $\gamma$, it suffices to lower bound the smallest eigenvalue of $B_S^T\tilde{C}(x)B_S$ for all $x \in U(x_\mu^*)$.
	Note that
	\[
		\min_{x \in U(x_\mu^*)}\lambda_{\min}(B_S^T \tilde{C}(x) B_S) ~ \ge ~
		\lambda_{\min} (B_S^T B_S) \min_{x \in U(x_\mu^*)} \min_{(i,j) \in \supp(Bx_\mu^*)} [\tilde{C}(x)]_{(i,j),(i,j)}.
	\]
	Let $s := (x(i)-x(j))^2 + \mu^2$, then by definition, $s \le |\Delta|^{2p-2}$ for all $x \in U(x_\mu^*)$ and for all $(i,j) \in \supp(Bx_\mu^*)$.
	Each diagonal entry in $\tilde{C}(x)$ is lower bounded by a function $h(s) = (q-1)s^{q/2-1}$,
	and since $h'(s) \le 0$ for all $s \ge 0$, we know that
	\[
		\min_{x \in U(x_\mu^*)} \min_{(i,j) \in \supp(Bx_\mu^*)} [\tilde{C}(x)]_{(i,j),(i,j)} ~ \ge ~
		\min_{0 \le s \le |\Delta|^{2p-2}} h(s) ~=~
		(q-1)\left(|\Delta|^{2p-2}\right)^{q/2-1} ~=~
		\frac{1}{(p-1)|\Delta|^{p-2}}.
	\]
	Finally, the result follows because
	\[
		\lambda_{\min} (B_S^T B_S) > \frac{1}{|\Delta|^2},
	\]
	which can be easily shown by using the local Cheeger-type result in~\cite{CHUNG200722}.
\end{proof}

We make two remarks about strong convexity. First, the same argument in the above proof also applies to $F_\mu^l(x)$ as $\nabla F_\mu(x) = \nabla F_\mu^l(x)$ for all $x \in U(x_\mu^*)$. The strong convexity result implies that the minimizers of $F_\mu(x)$ and $F_\mu^l(x)$ are unique. Second, the lower bound on the strong convexity parameter is pessimistic because we do not make any assumption about the internal conductance of the target cluster. Lower bounding $\gamma$ under stronger assumptions are beyond the scope of this paper. In practice, we observe much better performance, because most clusters are well connected internally. 

\begin{theorem}[Convergence]
\label{thm:convergence}
	Let $x_\mu^*$ denote the optimal solution for the globally smoothed problem~\eqref{eq:globalLq}, i.e.,
	\[
		x_\mu^* \ := \ \argmin_{x\ge0} F_\mu(x).
	\]
	The iterates $\{x_k\}_{k=0}^{\infty}$ generated by Algorithm~\ref{alg:coordinatedescent} converge to $x_\mu^*$. Moreover,
	\[
		x_\mu^* \ = \ \argmin_{x\ge0} F_\mu^l(x).
	\]
\end{theorem}

\begin{proof}
	Recall the optimality conditions of~\eqref{eq:globalLq},
	\begin{enumerate}[label=(\alph*)]
		\item Dual feasibility\footnote{We call $x\ge0$ dual feasibility because $F_\mu(x)$ smoothes the dual problem of $p$-norm flow diffusion.}: $x \ge 0$;
		\item Primal feasibility: $\nabla F_\mu(x) \ge 0$.
		\item Complementary slackness (under primal feasibility): $\nabla_i F_\mu(x) \le 0$ for every $i \in \supp(x)$. 
	\end{enumerate}
    	By the definition of index set $S_k$ in Algorithm~\ref{alg:coordinatedescent}, 
	the iterates $x_0$, $x_1$, $x_2$, $\ldots$ are monotone,
	i.e., $0 \le x_1 \le x_2 \le \ldots$, so $x_k$ satisfies item (a) for all $k$.
	By Lemma~\ref{lemma:boundedgradients}, $x_k$ also satisfies item (c), for all $k$.
	We may assume $S_k \neq \emptyset$ for all $k$, as otherwise Algorithm~\ref{alg:coordinatedescent} terminates with the optimal solution that satisfies item (b).
	So assume that $S_k \neq \emptyset$ for all $k$, we argue that the sequence $\{x_k\}_{k=0}^{\infty}$ converges.
	Suppose not, then by Lemma~\ref{lemma:lipschitz},
	\[
		x_{k+1} := x_k - \tfrac{\mu^{2-q}}{\deg(i_k)}\nabla_{i_k} F_\mu(x_k) e_{i_k}
		= x_k - \tfrac{1}{L_{i_k}}\nabla_{i_k} F_\mu(x_k) e_{i_k},
	\]
	for all $k$, and hence by coordinate Lipschitz continuity we get
	\[
		F_\mu(x_{k+1}) \le F_\mu(x_k) - \tfrac{1}{2L_{i_k}} \left(\nabla_{i_k} F_\mu(x_k)\right)^2,
	\]
	for all $k$. Because the sequence of iterates $\{x_k\}_{k=0}^{\infty}$ does not converge, 
	the sequence of gradients $\{\nabla_{i_k} F_\mu(x_k)\}_{k=0}^{\infty}$ must be bounded away from zero.
	This means that the sequence of function values $\{F_\mu(x_k)\}_{k=0}^{\infty}$ does not converge, either.
	But then this implies that $F_\mu(x)$ is unbounded below, contradicting our assumption that an optimal solution to~\eqref{eq:globalLq} exists.
	Therefore the sequence $\{x_k\}_{k=0}^{\infty}$ must converge to a limit point $\bar{x}$. 
	Now, because $\nabla F_\mu(x)$ is continuous, and since each $x_k$ satisfies optimality conditions (a) and (c), 
	$\bar{x}$ must also satisfy items (a) and (c).
	It is easy to see that $\bar{x}$ satisfies item (b), too, 
	because otherwise there must be some $x_k$ where $x_k(i) > \bar{x}(i)$ for some coordinate $i$, which is not possible.
	Thus, $\bar{x}$ is a minimizer of $F_\mu(x)$, and by uniqueness we have that $\bar{x} = x^*_{\mu}$.
	
	Furthermore, for all $x$ such that $\supp(x) \subseteq \supp(x_\mu^*)$,
	we have that $F_\mu^l(x) = F_\mu(x) - \tfrac{1}{q}\mu^q(|E|-|\supp(Bx_\mu^*)|)$, 
	that is,  $F_\mu(x)$ and $F_\mu^l(x)$ only differ by a constant. Hence $\nabla F_\mu^l(x) = \nabla F_\mu(x)$. 
	But this means that $x^*_{\mu}$ must also satisfy the optimality conditions of the locally smoothed problem~\eqref{eq:localLq}.
	This means $x_\mu^*$ is also the optimal solution of~\eqref{eq:localLq}.
\end{proof}

A direct result of Theorem~\ref{thm:convergence} is that we can use $F_\mu(x)$ and $F_\mu^l(x)$ interchangeably in the iteration complexity and running time analysis of Algorithm~\ref{alg:coordinatedescent}.

\begin{corollary}[Interchangeability]
\label{cor:interchangeability}
	Let $\{x_k\}_{k=0}^{\infty}$ be any sequence of iterates generated by Algorithm~\ref{alg:coordinatedescent}.
	Let $F_\mu^*$ and $F_\mu^{l*}$ be optimal objectives values of \eqref{eq:globalLq} and \eqref{eq:localLq}, respectively.
	Let $X := \{x_k\}_{k=0}^{\infty} \cup \{x_\mu^*\}$.
	Then for any $x \in X$, we have that
	\[
		F_\mu(x) - F_\mu^* = F_\mu^l(x) - F_\mu^{l*}.
	\]
\end{corollary}
\begin{proof}
	It immediately follows by noting that $F_\mu(x)$ and $F_\mu^l(x)$ differ by a constant value, i.e., 
	$F_\mu^l(x) = F_\mu(x) - \tfrac{1}{q}\mu^q(|E|-|\supp(Bx_\mu^*)|)$, for all $x \in X$.
\end{proof}

The convergence of monotonic iterates generated by Algorithm~\ref{alg:coordinatedescent} also guarantees that $F_\mu(x)$ is strongly convex on the iterates. 

\begin{corollary}[Strong convexity on iterates]
\label{cor:strongconvalgo}
	Let $\{x_k\}_{k=0}^{\infty}$ be any sequence of iterates generated by Algorithm~\ref{alg:coordinatedescent}. 
	If $\mu \le 1$, then $F_\mu(x)$ is strongly convex on $X := \{x_k\}_{k=0}^{\infty} \cup \{x_\mu^*\}$.
\end{corollary}

\begin{proof}
	By Lemma~\ref{lemma:strongconvexity}, we just need to verify that $X \subseteq U(x_\mu^*)$.
	Following Theorem~\ref{thm:convergence}, 
	the sequence $\{x_k\}_{k=0}^{\infty}$ is monotonically increasing and converges to $x_\mu^*$,
	so $\supp(x_k) \subseteq \supp(x_\mu^*)$ for all $k$.
	It suffices to show that
	\[
		(x(i)-x(j))^2 + \mu^2 \le |\Delta|^{2p-2}, \quad\forall(i,j) \in E, \quad\forall x \in X.
	\]
	Using the combinatorial interpretation provided in Figure~\ref{fig:comb_alg}, we obtain a trivial upper bound on the amount of flow, $\flow_q(i,j;x_k,\mu)$, on any edge $(i,j)$ at any iteration $k$, as follows. Let $k \ge 0$ be arbitrary. By Lemma~\ref{lemma:boundedgradients}, we know that $\ex(i; x_k, \mu) \ge 0$ for every node $i$ that has nonzero height $x_k(i) > 0$. This means that at iteration $k$ when we send mass from node $i_k$ to its neighbors, we never remove more mass from $i_k$ than its current excess $\ex(i_k; x_k, \mu)$. So $m(i; x_k, \mu) \ge 0$ for all node $i$ for all $k$. On the other hand, since the directions of flow, i.e., $\sign(\flow_q(i,j;x_k, \mu))$, are completely determined by the ordering of incumbent heights $x_k$, we know that flows cannot cycle. This is because if there is a directed cycle in the induced sub-graph on $\supp(Bx_k)$, where edges are oriented according to the directions of flow, then it means that there must exist a set of heights $\{x_k(i_j)\}$ where $x_k(i_1) < x_k(i_2) < \ldots < x_k(i_1)$, which is not possible. Now, since all nodes have non-negative mass and there is no cycling of flows, the net flow on any edge $(i,j) \in E$ cannot be larger than the total amount of initial mass $|\Delta|$ minus one (one is the lowest possible degree of a node in the underlying connected graph),
\begin{equation}
\label{eq:flowbound}
	((x_k(i)-x_k(j))^2 + \mu^2)^{q/2-1}|x_k(i)-x_k(j)| = |\flow_q(i,j;x_k,\mu)| \le |\Delta| - 1, \quad\forall~ (i,j) \in E.
\end{equation}
	
	We use \eqref{eq:flowbound}and the assumption $\mu \le 1$ to get that, for all $x \in X$ and for all $(i,j) \in E$,
	\[
	\begin{split}
		&~((x(i)-x(j))^2 + \mu^2)^{(q-1)/2}\\
		=&~((x(i)-x(j))^2 + \mu^2)^{q/2-1}((x(i)-x(j))^2 + \mu^2)^{1/2} \\
		\le&~((x(i)-x(j))^2 + \mu^2)^{q/2-1}(|x(i)-x(j)| + \mu)\\
		=&~((x(i)-x(j))^2 + \mu^2)^{q/2-1}|x(i)-x(j)| + ((x(i)-x(j))^2 + \mu^2)^{q/2-1}\mu \\
		\le&~((x(i)-x(j))^2 + \mu^2)^{q/2-1}|x(i)-x(j)| + (\mu^2)^{q/2-1}\mu \\
		=&~((x(i)-x(j))^2 + \mu^2)^{q/2-1}|x(i)-x(j)| + \mu^{q-1} \\
		\le&~|\Delta| - 1 + 1\\
		=&~|\Delta|,
	\end{split}
	\]
	and thus we have
	\[
		(x(i)-x(j))^2 + \mu^2 \le |\Delta|^{2/(q-1)} = |\Delta|^{2p-2},
	\]
	as required.
\end{proof}

Lipschitz continuity, strong convexity, and local smooth approximation bound~\eqref{eq:smoothapprox} give us the following convergence rate guarantee.

\begin{theorem}[Iteration complexity]
\label{thm:iterationcomplexity}
	Let $x_\mu^*$ and $F_\mu^*$ denote the optimal solution and optimal value of~\eqref{eq:globalLq}, respectively.
	For any $k \ge 1$, let $x_k$ denote the $k$th iterate generated by Algorithm~\ref{alg:coordinatedescent}.
	Let $\gamma$ be the strong convexity parameter as described in Lemma~\ref{lemma:strongconvexity},
	and let $L_{\mu} = \max_{i \in \supp(x_\mu^*)} L_i$, where $L_i = \deg(i)\mu^{q-2}$ be the coordinate Lipschitz constants of $F_\mu(x)$.
	After $K = \mathcal{O}\big(\tfrac{|\Delta|L_{\mu}}{\gamma}\log \tfrac{1}{\epsilon}\big)$ iterations, one has
	\[
		\mathbb{E}[F_\mu(x_K)] - F_\mu^* \le \epsilon.
	\]
	Furthermore, let $F^*$ denote the optimal objective value of~\eqref{eq:dualLqreg}. 
	If we pick $\mu = \mathcal{O}\big(\big(\tfrac{\epsilon}{|\Delta|}\big)^{1/q}\big)$,
	then after 
	\[
		K' = \mathcal{O}\left(\frac{|\Delta|\bar{d}}{\gamma}\left(\frac{|\Delta|}{\epsilon}\right)^{2/q-1} \log\frac{1}{\epsilon}\right)
	\]
	iterations, where $\bar{d} = \max_{i \in \supp(x_\mu^*)}\deg(i)$, one has 
	\[
		\mathbb{E}[F(x_{K'})] - F^* \le \epsilon.
	\]
\end{theorem}

\begin{proof}
	Using coordinate Lipschitz constants given in Lemma~\ref{lemma:lipschitz} we have that
	\[
		F_\mu(x_{k+1}) = F_\mu(x_k - \tfrac{1}{L_i} \nabla_i F_\mu(x_k) e_i)
		\le F_\mu(x_k) - \tfrac{1}{L_i} \nabla_i F_\mu(x_k)^2 + \tfrac{1}{2L_i}\nabla_i F_\mu(x_k)^2
  		= F_\mu(x_k) - \tfrac{1}{2L_i}\nabla_i F_\mu(x_k)^2.
	\]
	Let $L_{S_k}:= \max_{i\in S_k} L_i$, and take conditional expectation
	\begin{align*}
		\mathbb{E}\left[F_\mu(x_{k+1}) \ | \ x_k\right]
		\le~&  F_\mu(x_k) - \tfrac{1}{2} \sum_{i \in V} \tfrac{1}{|S_k|} \tfrac{1}{L_i}\nabla_i F_\mu(x_k)^2 \\ 
		=~&   F_\mu(x_k) - \tfrac{1}{2 |S_k|} \sum_{i\in S_k} \tfrac{1}{L_i}\nabla_i F_\mu(x_k)^2 \\
		\le~& F_\mu(x_k) - \tfrac{1}{2 |S_k| L_{S_k}} \|\nabla_{S_k} F_\mu(x_k)\|_2^2 \\
		\le~& F_\mu(x_k) - \tfrac{1}{2 |S_k| L_{S_k}} \|\nabla F_\mu(x_k)\|_2^2.
	\end{align*}
	From strong convexity of $F_\mu(x)$ on $\{x_k\}_{k=0}^{\infty} \cup \{x^*_{\mu}\}$ (Lemma~\ref{lemma:strongconvexity} and Corollary~\ref{cor:strongconvalgo}) we have that 
	\[
		\|\nabla F_\mu(x_k)\|_2^2 \le 2 \gamma (F_\mu(x_k) - F_\mu^*).
	\]
	Therefore, we get 
	\[
		\mathbb{E}\left[F_\mu(x_{k+1}) \ | \ x_k\right] \le F_\mu(x_k)  - \tfrac{\gamma}{|S_k| L_{S_k}} (F_\mu(x_k) - F_\mu^*),
	\]
	and so
	\[
		\mathbb{E}\left[F_\mu(x_{k+1}) \ | \ x_k\right]  - F_\mu^* \le \left(1 - \tfrac{\gamma}{|S_k| L_{S_k}} \right)(F_\mu(x_k) - F_\mu^*).
	\]
	Note that $S_k \subseteq \supp(x_\mu^*)$, thus $|S_k| \le |\supp(x_\mu^*)| \le |\supp(Bx_\mu^*)| \le |\Delta|$. 
	Let $L_{\mu}:= \max_{i\in \supp(x_\mu^*)} L_i$, then $L_{\mu} \ge L_{S_k}$. Using these simple inequalities we get
	\[
		\mathbb{E}\left[F_\mu(x_{k+1}) \ | \ x_k\right]  - F_\mu^* \le \left(1 - \tfrac{\gamma}{|\Delta| L_{\mu}} \right)(F_\mu(x_k) - F_\mu^*).
	\]
	Take conditional expectations over all $x_{k-1}$, $x_{k-2}$, $\ldots$, $x_1$, $x_0$ we get
	\[
		\mathbb{E}\left[F_\mu(x_{k+1}) \right]  - F_\mu^* \le \left(1 - \tfrac{\gamma}{|\Delta| L_{\mu}} \right)^k(F_\mu(x_0) - F_\mu^*).
	\]
	Therefore, after $K = \mathcal{O}\big(\tfrac{|\Delta|L_{\mu}}{\gamma}\log \tfrac{1}{\epsilon}\big)$ iterations, one has
	\[
		\mathbb{E}[F_\mu(x_k)] - F_\mu^* \le \tfrac{\epsilon}{2}.
	\]
	Using Corollary~\ref{cor:interchangeability} we get that, after $K$ iterations,
	\[
		\mathbb{E}[F_\mu^l(x_k)] - F_\mu^{l*}  = \mathbb{E}[F_\mu(x_k)] - F_\mu^* \le \tfrac{\epsilon}{2},
	\]
	where $F_\mu^{l*}$ is the optimal objective value of \eqref{eq:localLq}. 
	It then follows from Lemma~\ref{lemma:smoothapprox} that
	\[
		\mathbb{E}[F(x_K)] - F^* \ \le \ \mathbb{E}[F_\mu^l(x_K)] - F_\mu^{l*} + \tfrac{1}{q}\mu^q|\Delta| \ \le \ \tfrac{\epsilon}{2} + \tfrac{1}{q}\mu^q|\Delta|.
	\]
	Hence setting $\mu = \mathcal{O}\big(\big(\tfrac{\epsilon}{|\Delta|}\big)^{1/q}\big)$ gives the required iteration complexity.
\end{proof}

\begin{corollary}[Running time]	
\label{cor:runtime}
	If we pick $\mu = \mathcal{O}\big(\big(\tfrac{\epsilon}{|\Delta|}\big)^{1/q}\big)$, 
	the total running time of Algorithm~\ref{alg:coordinatedescent} to obtain an $\epsilon$ accurate solution of \eqref{eq:dualLqreg} is 
	$\mathcal{O}\big(\tfrac{|\Delta|\bar{d}^2}{\gamma}\big(\frac{|\Delta|}{\epsilon}\big)^{2/q-1} \log\frac{1}{\epsilon}\big)$.
\end{corollary}

\begin{proof}
	This is straightforward by noticing that at each step in Algorithm~\ref{alg:coordinatedescent}, we touch only the nodes $j$ such that $j \sim i_k$, 
	for updating gradient vector to $\nabla F_\mu(x_{k+1})$ and for obtaining $S_{k+1}$.
\end{proof}

Finally, we remark that Algorithm~\ref{alg:coordinatedescent} is easily parallelizable, as all coordinates in $S_k$ can be updated at the same time, without sacrificing locality.

\section{Empirical Set-up and Results}

\subsection{Computing platform and implementation detail}

We implemented Algorithm~\ref{alg:coordinatedescent} in Julia\footnote{Our code is available at \url{http://github.com/s-h-yang/pNormFlowDiffusion}.}. When $p = q = 2$, the objective function of the regularized $q$-norm cut problem~\eqref{eq:dualLqreg} has coordinate Lipschitz constants $L_i = \deg(i)$, therefore we can directly solve~\eqref{eq:dualLqreg} in linear and strongly local running time. As discussed earlier, coordinate methods enjoy a natural combinatorial interpretation as routing mass in the underlying graph. Algorithm~\ref{alg:combinatorialL2} provides a direct specialization of Algorithm~\ref{alg:coordinatedescent} to the case $q=2$, where we describe the algorithmic steps in the equivalent combinatorial setting as diffusing excess mass in the graph.

{\centering
\begin{minipage}{.8\linewidth}
\begin{algorithm}[H]
	\noindent\fbox{
		\begin{varwidth}{\dimexpr\linewidth-2\fboxsep-2\fboxrule\relax}
			\begin{algorithmic}
				\STATE 
					\begin{enumerate}[itemsep=0mm]
						\item Initially, $x(v)  = 0$ and $\ex(v) = \max\{\Delta(v) - \deg(v), 0\}$ for all $v \in V$.
						\item While $\ex(v) > 0$ for some node $v$:
							\begin{enumerate}[label=(\alph*),itemsep=0mm]
								\item Pick any $v$ where $\ex(v) > 0$.
								\item Apply $\texttt{push}(v)$.
							\end{enumerate}
						\item Return $x$.
					\end{enumerate}
			\end{algorithmic}
		\end{varwidth}
	}
	\noindent\fbox{
		\begin{varwidth}{\dimexpr\linewidth-2\fboxsep-2\fboxrule\relax}
			\begin{algorithmic}
				\STATE $\texttt{push}(v)$:
				\vspace{2mm}
				\STATE Make the following updates:
					\begin{enumerate}[itemsep=0mm]
						\item $x(v) \gets x(v) + \ex(v)/\deg(v)$.
						\item $\ex(v) \gets 0$.
						\item For each node $u \sim v$:
							~$\ex(u) \gets \ex(u)+ \ex(v)/\deg(v)$.
					\end{enumerate}
			\end{algorithmic}
		\end{varwidth}
	}
	\caption{Coordinate solver for \eqref{eq:dualLqreg} when $q=2$}
	\label{alg:combinatorialL2}
\end{algorithm}
\end{minipage}
\par}
\vspace{5mm}

For general $p > 2$ and $1 < q < 2$, our implementation adds an additional line-search step. More specifically, instead of using the fixed step-sizes $\mu^{2-q}/\deg(i_k)$ given in Algorithm~\ref{alg:coordinatedescent}, we use binary search to find step-sizes $\alpha_k$ such that
\[
	\nabla_{i_k}F_\mu\big(x_k - \alpha_k\nabla_{i_k}F_\mu(x_k)e_{i_k}\big) = 0.
\]
This leads to coordinate minimization steps that can improve practical convergence. Computing the required step-sizes $\alpha_k$ through binary line-search is possible because the partial gradients are monotone (cf. Lemma~\ref{lemma:monotonegradient}).

Finally, for efficient implementation that avoids iteratively sampling {\em with} replacement the indices for coordinate updates, we adopt a sampling {\em without} replacement approach that is seen in random-permutation cyclic coordinate updates~\cite{LW2018}. That is, every time an index set $S_k$ is constructed, we loop over all coordinates in $S_k$ randomly without replacement, before computing a new index set $S_{k+1}$.

\subsection{Diffusion on a dumbbell}

The best way to visualize $p$-norm flow diffusions for various $p$ values is to start the diffusion processes on the same graph and the same set of seed nodes, with equal amount of initial mass. To this end, we run $p$-norm diffusions for $p \in \{2,4,8\}$ on a synthetic tiny ``dumbbell'' graph obtained by removing edges from a $7 \times 7$ grid graph. We pick a single seed node which locates on one side of the ``bridge'', and set $|\Delta| = 121$. For each $p$, we plot optimal dual variables in Figure~\ref{fig:dumbbell}, where we use color intensities and circle sizes to indicate the relative magnitude of dual values, i.e., brighter colors and larger circles size represent higher dual values for a fixed $p$, and no circle means the corresponding node has zero dual value.

\begin{figure}[ht]
	\centering
	\begin{subfigure}{.33\textwidth}
  		\centering
  		\includegraphics[width=.75\textwidth]{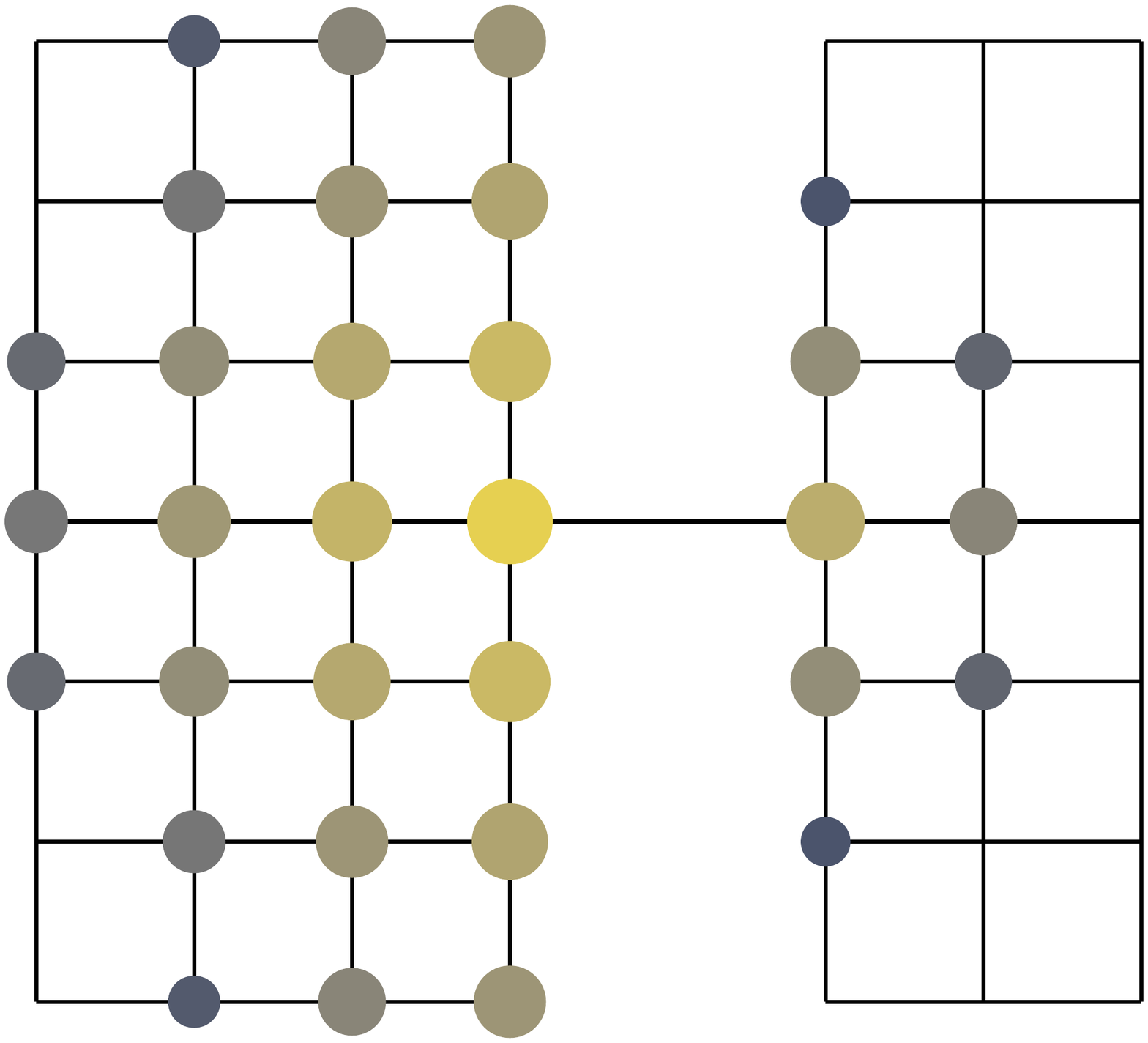}
		\caption{2-norm diffusion}
	\end{subfigure}%
	\begin{subfigure}{.33\textwidth}
  		\centering
  		\includegraphics[width=.75\textwidth]{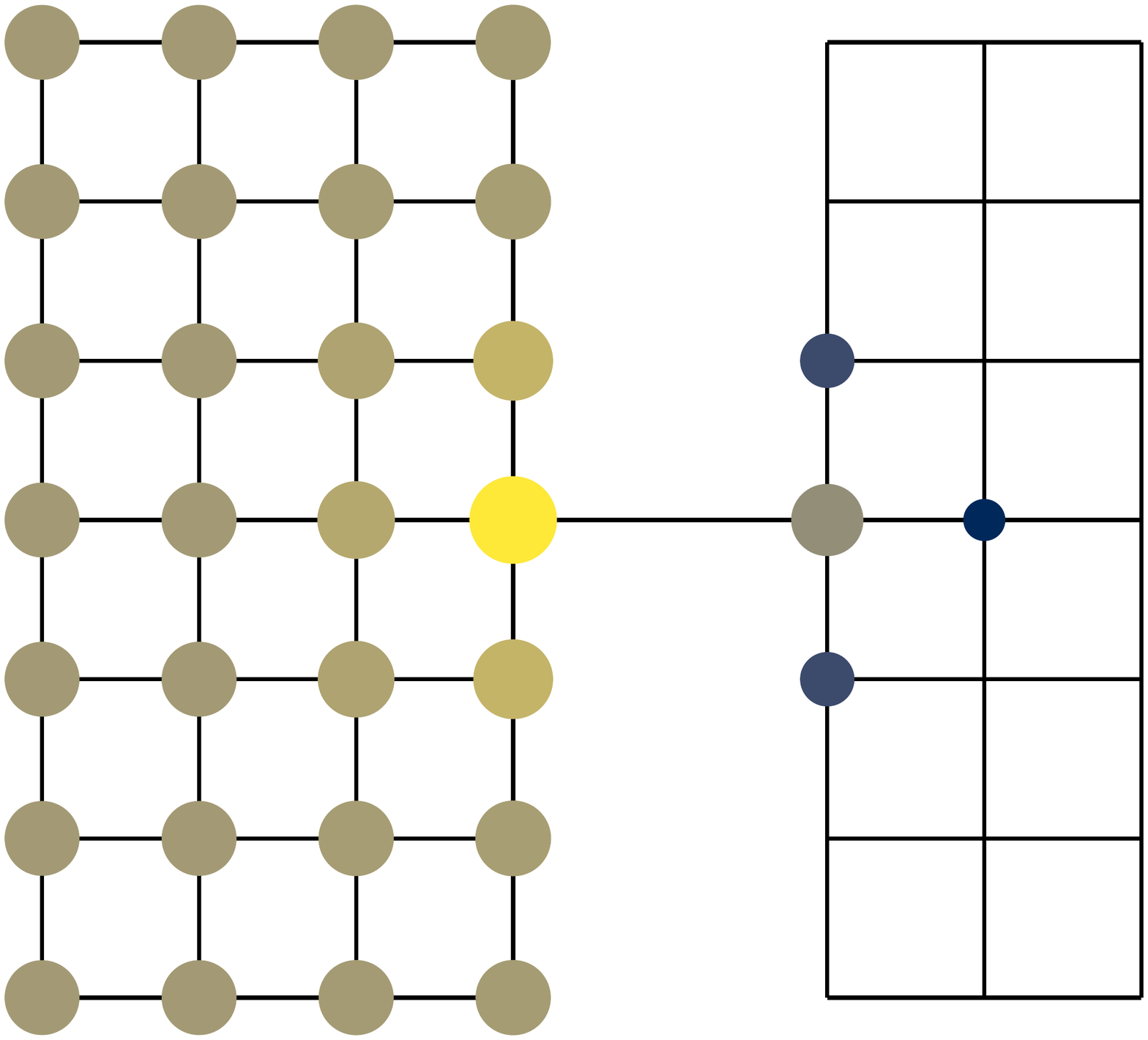}
		\caption{4-norm diffusion}
	\end{subfigure}%
	\begin{subfigure}{.33\textwidth}
  		\centering
  		\includegraphics[width=.75\textwidth]{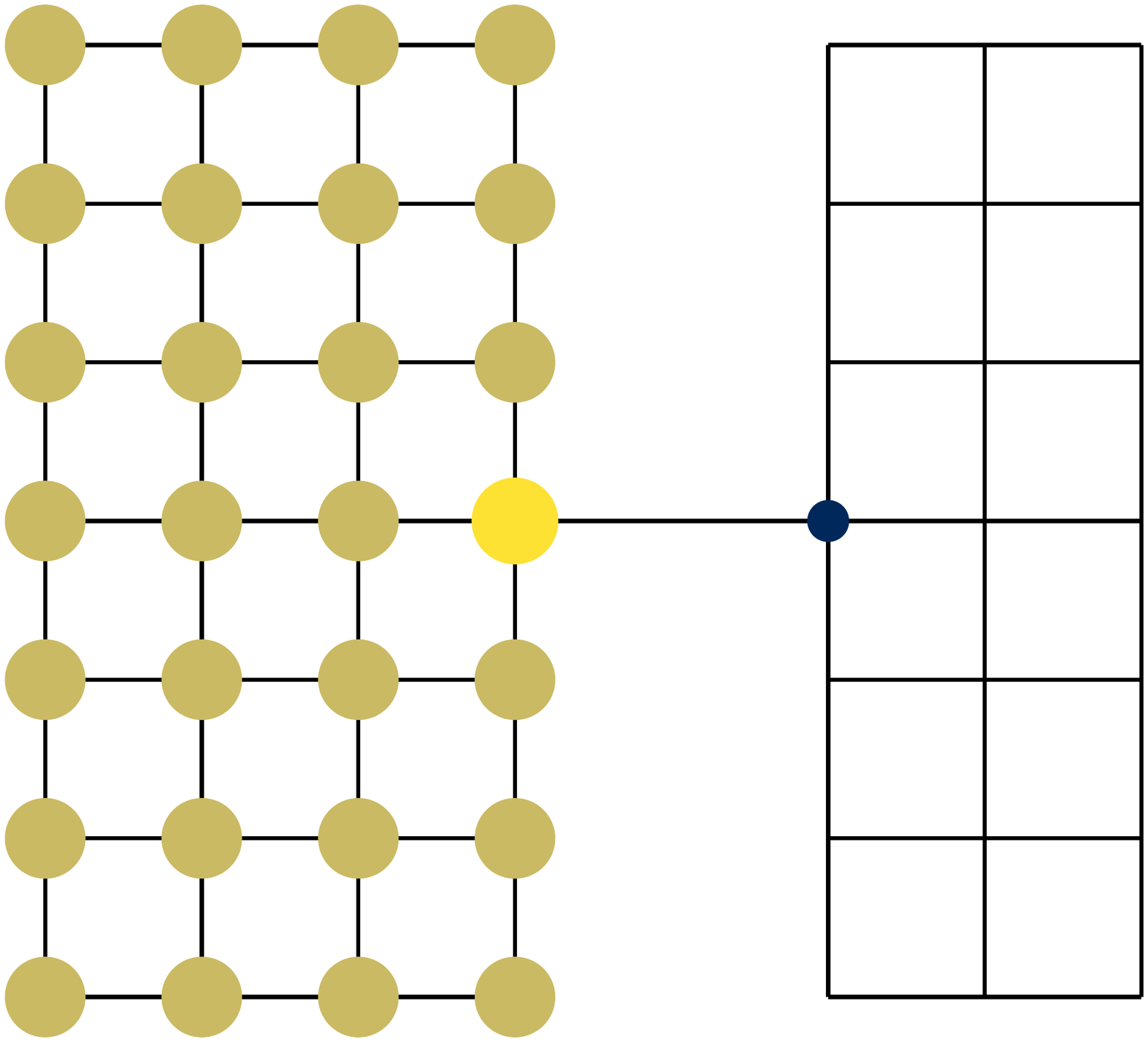}
		\caption{8-norm diffusion}
	\end{subfigure}
	\caption{Diffusion on a dumbbell: color intensities and circle sizes are chosen to reflect relative magnitude of optimal dual variables}
    	\label{fig:dumbbell}
\end{figure}

Recall that a dual variable is nonzero only if the corresponding node is saturated (i.e., the mass it holds equals its degree). Observe that 2-norm diffusion leaks a lot of mass to the other side of dumbbell, whereas 4-norm and 8-norm diffusions saturate entire left-hand side, without leaking much mass to the right. The reason that this happens is because $p=4$ and $p=8$ put significantly larger penalty on the flow that passes through the ``bridge'', making it difficult to send mass over to the other side. 

Note that, if we were to perform the sweep cut procedure described in Section~\ref{sec:clustering}, then 2-norm diffusion would fail to recover the ``correct'' cluster, while both 4-norm and 8-norm diffusions return the entire left-hand side of dumbbell, which is the best possible result in terms of low conductance. Although this example is overly simplistic, it does demonstrate that $p$-norm flow diffusion with higher $p$ values are more sensitive to bottlenecks when routing mass around a seed node, and, on the other hand, it is possible to overcome such bottleneck even when $p$ is slightly larger than 2, say $p = 4$ and $p = 8$. Indeed, our subsequent experiments show that there is already a significant improvement in the context of local graph clustering, when raising $p=2$ to $p=4$.

\subsection{Experiments}

The goal of our experiments is two-fold. First, we carry out experiments on various LFR synthetic graphs~\cite{LFR08} and demonstrate that the theoretical guarantees of $p$-norm flow diffusion are well reflected in practice. Second, we show the advantage of $p$-norm diffusion for local graph clustering tasks on four real datasets consisting of both social and biological networks. We compare the performance of $p$-norm flow diffusion with $\ell_1$-regularized PageRank~\cite{FKSCM2017} and nonlinear diffusion under power transfer~\cite{ID19}. 
Although our theoretical analysis holds for $p\in (1,\infty)$ and Algorithm~\ref{alg:coordinatedescent} converges linearly for any $p \ge 2$, we think in practice the most interesting regime is when $p$ is a small constant, e.g. $p\in [2,8]$, as our theory suggests that the marginal gain in terms of conductance guarantee diminishes as $p$ grows large. We elaborate this in the experiment on LFR synthetic graphs by comparing marginal improvements in the performance by raising $p$ from 2 to 4, and from 4 to 8.

\subsubsection{Datasets}

The LFR model~\cite{LFR08} is a widely used benchmark for evaluating community detection algorithms. It is essentially a stochastic block model with the additional property that nodes' degrees follow the power law distribution, and there is a parameter $\mu$ controlling what fraction of a node's neighbours is outside the node's block. We use synthetic LFR graphs generated under the following parameter setting: number of nodes is 1000, average degree is 10, maximum degree is 50, minimum community size is 20, maximum community size is 100, power law exponent for degree distribution is 2, power law exponent for community size distribution is 1. We consider different $\mu$ values between 0.1 and 0.4 with step 0.02. This range means that the synthetic graphs contain reasonable noisy (but not completely noise) clusters. 

The four real datasets include both social and biological networks. The graphs that we consider are all unweighted and undirected. Table~\ref{tab:sumdata} shows some basic characteristics of these graphs. 

\begin{table}[h]
	\caption{Summary of real-world graphs}
	\label{tab:sumdata}
	\vspace{2mm}
	\centering
	\begin{adjustbox}{max width=\textwidth}
	\begin{tabular}{cccl}
	dataset & number of nodes & number of edges &  \multicolumn{1}{c}{\rotatebox[origin=c]{0}{description}}\\
	\midrule
	FB-Johns55 & $5157$ & $186572$ & Facebook social network for Johns Hopkins University\\   
	Colgate88 & $3482$ & $155043$ & Facebook social network for Colgate University\\   
	Sfld & $232$ & $15570$ & Pairwise similarities of blasted sequences of proteins\\    
	Orkut & $3072441$ & $117185083$ & Large-scale on-line social network\\   
	\midrule
	\end{tabular}
	\end{adjustbox}
\end{table}

The two Facebook graphs are chosen from the Facebook100 dataset based on their assortativity values in the first column of Table A.2 in~\cite{TMP2012}, where the data were first introduced and analyzed. Each of these comes with some features, e.g., gender, dorm, major index and year. We consider a set of nodes with the same feature as a ground truth cluster, and filter the ``ground truth'' clusters by setting a 0.6 threshold on maximum conductance. We also omit clusters whose volume is larger than one third of the volume of the entire graph, since discovering clusters whose sizes are close to the entire graph is not the purpose of local clustering.

The biology dataset Sfld contains pairwise similarities of blasted sequences of 232 proteins belonging to the amidohydrolase superfamily~\cite{brown2006gold}. A gold standard is provided describing families within the given superfamily. According to the gold standard the amidrohydrolase superfamily contains 29 families. We consider each family as a ground truth cluster, and filter them by setting a 0.9 threshold on the maximum conductance. 

The last dataset Orkut is a free on-line social network where users form friendship each other. It can be downloaded from~\cite{snapnets}. This network comes with 5000 ground truth communities, which we filter by setting minimum community size 50, maximum cut conductance 0.5, and minimum ratio between internal connectivity (i.e., the smallest nonzero eigenvalue of the normalized Laplacian of the subgraph defined by the cluster) and cut conductance to 0.85. This resulted in 11 reasonably noisy clusters, having conductances between 0.4 and 0.5.

We include the statistics of all ground truth clusters that we used in the experiments in Table~\ref{tab:groundtruthclusters}.

\begin{table}[h]
    \caption{Filtered ``ground truth'' clusters for real-world graphs}
    \label{tab:groundtruthclusters}
    \vspace{2mm}
    \centering
    \begin{tabular}{ccccc}
    dataset & feature & volume & nodes  & conductance\\
     \midrule
     \multirow{5}{*}{\rotatebox[origin=c]{0}{FB-Johns55}} & year $2006$& $81893$ & $845$& $0.54$\\
										    & year $2007$& $89021$  & $842$ & $0.49$\\     										    
										    & year $2008$& $82934$  & $926$& $0.39$\\
     										    & year $2009$&  $33059$ & $910$& $0.21$\\
										    & major index $217$ & $10697$ & $201$ & $0.26$ \\
     \midrule
     \multirow{6}{*}{\rotatebox[origin=c]{0}{Colgate88}}  & year $2004$ & $14888$  &$230$ & 0.54\\
     										    & year $2005$ & $50643$  & $501$ & $0.50$\\
     										    & year $2006$ & $62065$  & $557$ & $0.48$\\
     										    & year $2007$ & $68382$  & $589$ & $0.41$\\
							                             & year $2008$ & $62430$  & $641$ & $0.29$\\
							                             & year $2009$ & $35379$  & $641$ & $0.11$\\						                             
    \midrule
    \multirow{6}{*}{\rotatebox[origin=c]{0}{Sfld}}		& urease 				& 31646 	& 100 	& 0.42 \\
     										& AMP 				& 3186 	& 28		& 0.53 \\
										& phosphotriesterase 	& 381 	& 7		& 0.78 \\
     										& adenosine 			& 1062	& 10		& 0.83 \\
										& dihydroorotase3 		& 494 	& 7 		& 0.83 \\						    
										& dihydroorotase2 		& 3119	& 13		& 0.90 \\
    \midrule
    \multirow{11}{*}{\rotatebox[origin=c]{0}{orkut}}	& A 	& 49767 	& 383 	& 0.42 \\
     										& B 	& 31912 	& 202 	& 0.45 \\
										& C 	& 16022 	& 141 	& 0.45 \\
										& D 	& 11698 	& 113 	& 0.46 \\
     										& E	& 26248 	& 194 	& 0.47 \\							    
										& F 	& 4617 	& 64 		& 0.47 \\	
										& G 	& 13786 	& 128 	& 0.47 \\
										& H	& 14109 	& 107 	& 0.48 \\
										& I 	& 18652 	& 195 	& 0.49 \\
     										& J 	& 41612 	& 318 	& 0.50 \\
										& K 	& 20204 	& 223 	& 0.50 \\
    \midrule
    \end{tabular}
\end{table}

\subsubsection{Methods and parameter setting}

We compare the performance of $p$-norm flow diffusion with $\ell_1$-regularized PageRank~\cite{FKSCM2017} and nonlinear diffusion under power transfer~\cite{ID19}. Given a starting node $v_s$, teleportation probability $\alpha$, and tolerance parameter $\rho$, the $\ell_1$-regularized PageRank is an optimization problem whose optimal solution is an approximate personalized PageRank (APPR) vector~\cite{ACL06}. This $\ell_1$-regularized variational formulation allow us to apply coordinate method and obtain an APPR vector in linear and strongly local running time. The nonlinear diffusion model iteratively applies a point-wise nonlinear activation to node values after each matrix-vector product between the Laplacian matrix and the incumbent node vector. Since it is demonstrated in~\cite{ID19} that a nonlinear transfer defined by the power function $u \mapsto u^{0.5}$ has the best overall performance when compared to other nonlinear functions like $\tanh$ or heat kernel, we only compare $p$-norm flow diffusion with nonlinear diffusion governed by this power function. We choose $p \in \{2,4\}$ for $p$-norm diffusion and demonstrate the advantage of our method even when $p$ is a small constant.

Our goal here is to compare the behaviour of different algorithms under a unified setting, and not to fine tune any particular model. Therefore, besides an additional experiment on the LFR synthetic graphs that examines the effect of having a larger overlap between an initial set of seed nodes and the target cluster,  we always start $p$-norm flow diffusion from a single seed node, and set $|\Delta| = t \cdot \vol(C)$ for some constant factor $t$ and some target cluster $C$ (recall from Assumption~\ref{assum:source} this is WLOG). In particular, on the Facebook datasets we set $t = 3$ because the target clusters already have a large volume, and for both the LFR and Orkut datasets we set $t=5$. On the other hand, because the clusters in Sfld are very noisy, we vary $t \in \{1,2,\ldots,10\}$ and pick the cluster with the lowest conductance. In all cases, we make sure the choice of $t$ is such that $|\Delta|$ is less than the volume of the whole graph. For the nonlinear diffusion model with power transfer, we use the same parameter setting as what the authors suggested in~\cite{ID19}. The $\ell_1$-regularized PageRank is the only linear model among all methods that we compare, so we allow it to use ground truth information to choose the teleportation parameter $\alpha$ giving the best conductance result. We tune $\alpha$ by picking $\alpha \in \{\lambda/8, \lambda/4, \lambda/2, \lambda, 2\lambda\}$, where $\lambda$ is the smallest nonzero eigenvalue of the normalized Laplacian for the subgraph that corresponds to the target cluster. Since the support size of APPR vector is bounded by $1/\rho$~\cite{FKSCM2017}, and the support size of dual variables for $p$-norm diffusion is bounded by $|\Delta|$, for comparison purposes, we set the tolerance parameter $\rho$ for $\ell_1$-regularized PageRank so that $\rho = 1/|\Delta|$.

\subsubsection{Results on LFR synthetic graphs}

Our theory indicates (not surprisingly) that better overlap of the input seed set and a target cluster will result in output cluster having better conductance and F1 measure, and we demonstrate this empirically in our first experiment on an LFR synthetic graph with $\mu = 0.3$. Note that the parameter $\mu = 0.3$ means that 30\% of all edges of a node from some ground truth cluster links to the outside of that cluster. Because of this noise level, the goal is not to recover the ground truth exactly, but to obtain a cluster that overlaps well with the target (e.g., has a good F1 score). We have chosen $\mu = 0.3$ because it represents reasonably noisy clusters that are not completely noise. For this experiment, we randomly pick a set of seed nodes $S$ from some target cluster $C$ in the graph and vary the percentage overlap of $S$ in $C$, i.e., $|S|/|C|$. Then for each $p \in \{2,4,8\}$, we run $p$-norm flow diffusion and use the Sweep Cut procedure to find a smallest conductance cluster. Figure~\ref{fig:p_vs_overlap} shows the mean and the variance for the conductance and the F1 measure while varying the ratio $|S|/|C|$. As expected, as the percentage overlap $|S|/|C|$ increases, which also means that the tightest $\alpha$ in Theorem~\ref{thm:main} decreases, we recover clusters with lower conductances and higher F1 scores. Observe that while there is a large gap in both the conductance and the F1 measure between the results for $p=2$ and $p=4$, the gain in raising from $p=4$ to $p=8$ is comparatively marginal.

\begin{figure}[ht]
	\centering
	\begin{subfigure}{.5\textwidth}
  		\centering
  		\includegraphics[width=.9\textwidth]{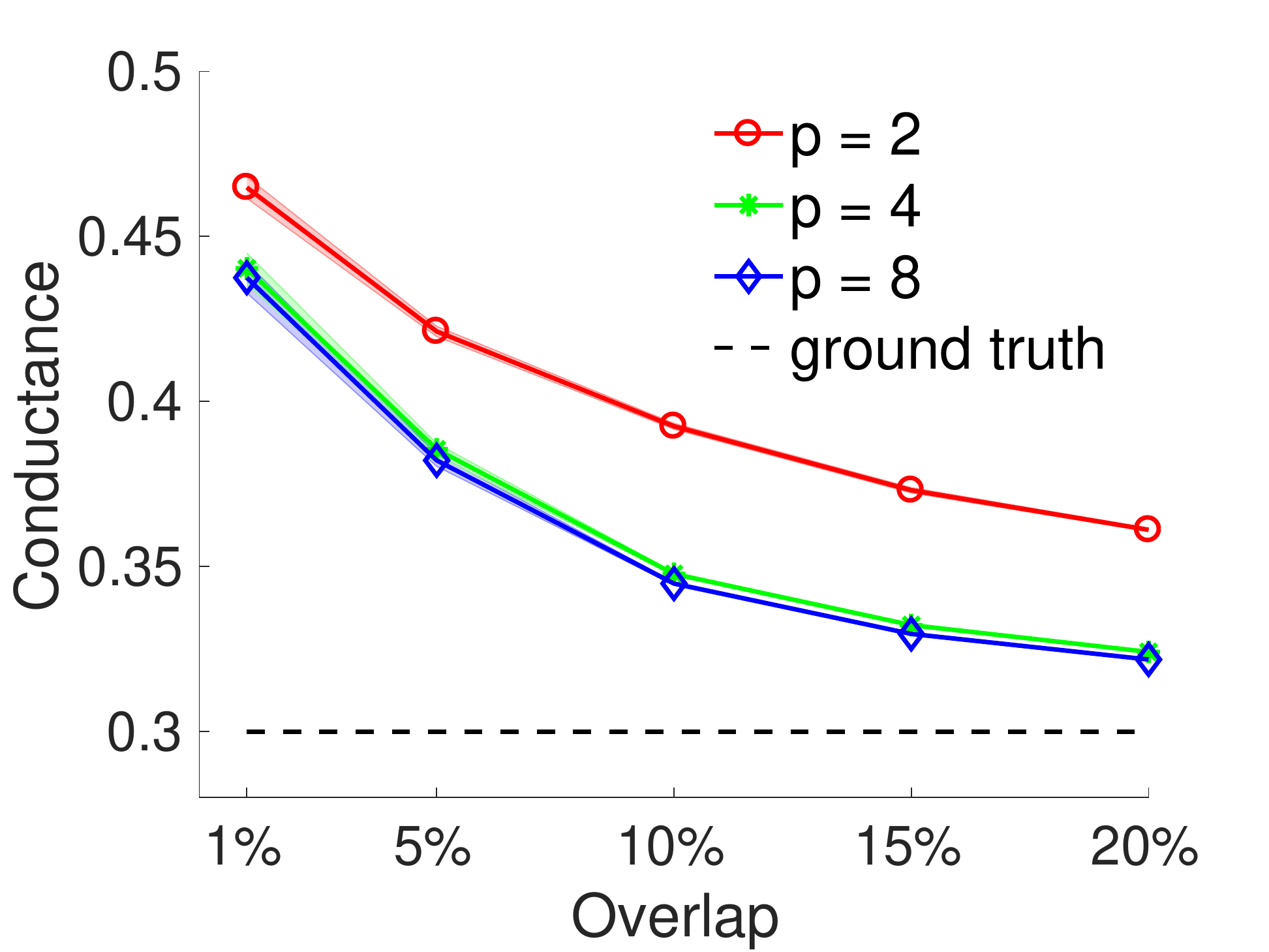}
	\end{subfigure}%
	\begin{subfigure}{.5\textwidth}
  		\centering
  		\includegraphics[width=.9\textwidth]{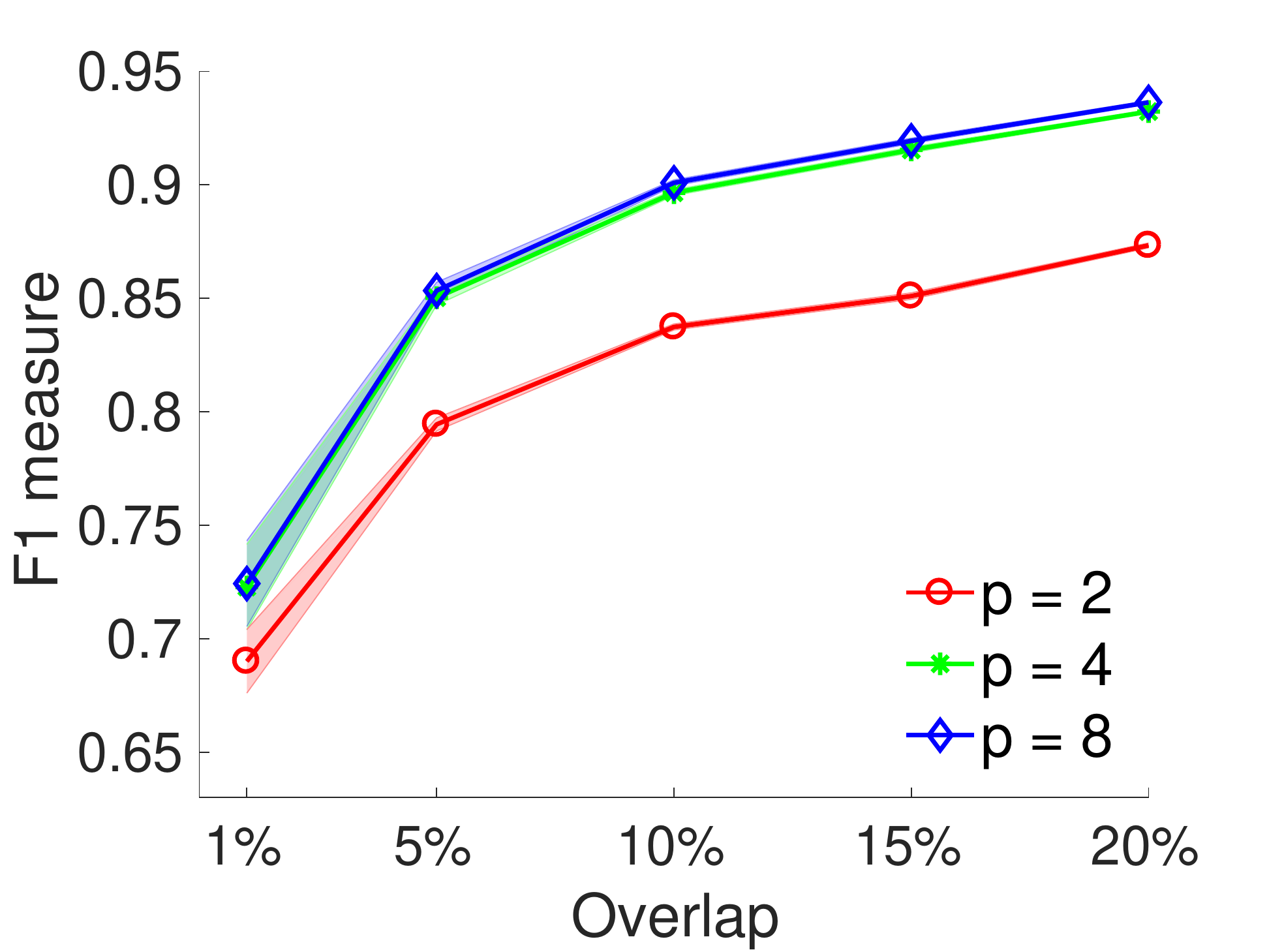}
	\end{subfigure}
	\caption{Conductance and F1 measure from different levels of initial overlap on an LFR synthetic graph. The black dashed line show ground truth conductance. The bands show the variation over 100 trials.}
	\label{fig:p_vs_overlap}
\end{figure}

In all subsequent experiments on both synthetic and real datasets when we compare recovery results with $\ell_1$-regularized PageRank and nonlinear power diffusion, we always start the diffusion process from one seed node, as this is the most common practice for semi-supervised local clustering tasks.

Our second experiment considers all LFR synthetic graphs. For each graph, we start from a random seed node and we repeat the experiment 100 times. Figure~\ref{fig:LFR} shows the mean and the variance for the conductance and the F1 measure while varying $\mu$. Notice that $p$-norm flow diffusion behaves similarly to the fine tuned $\ell_1$-regularized PageRank in both the conductance and the F1 measure when $p=2$, whereas it significantly outperforms other methods when $p=4$ and $p=8$. Observe that there is a slight gain in terms of conductance by raising $p=4$ to $p=8$, but such improvement is marginal. This is not really surprising, since qualitatively the $4$-norm unit ball is already very close to the $\infty$-norm unit ball (i.e. the box).

\begin{figure}[ht]
	\centering
	\begin{subfigure}{.5\textwidth}
  		\centering
  		\includegraphics[width=.9\textwidth]{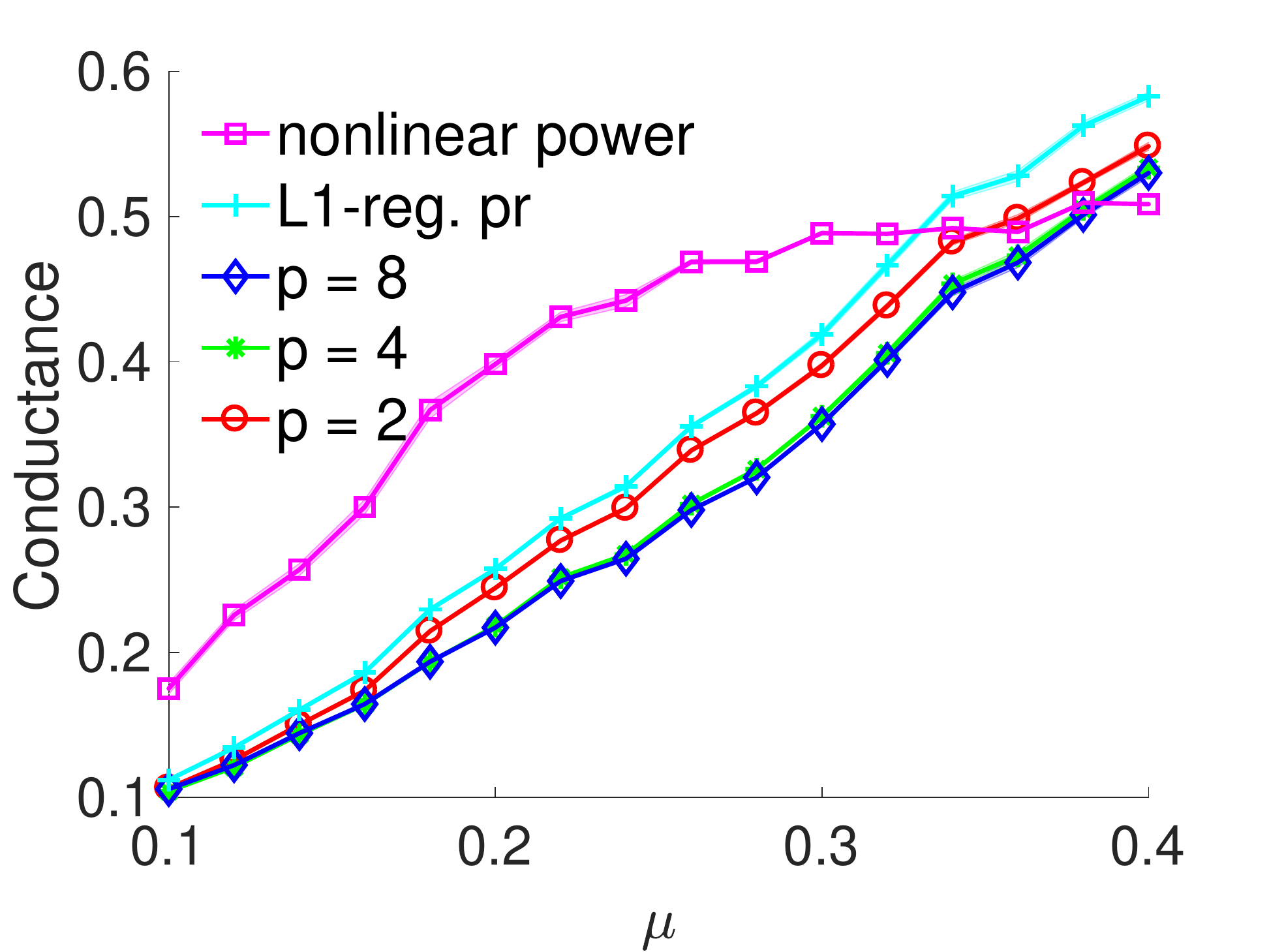}
	\end{subfigure}%
	\begin{subfigure}{.5\textwidth}
  		\centering
  		\includegraphics[width=.9\textwidth]{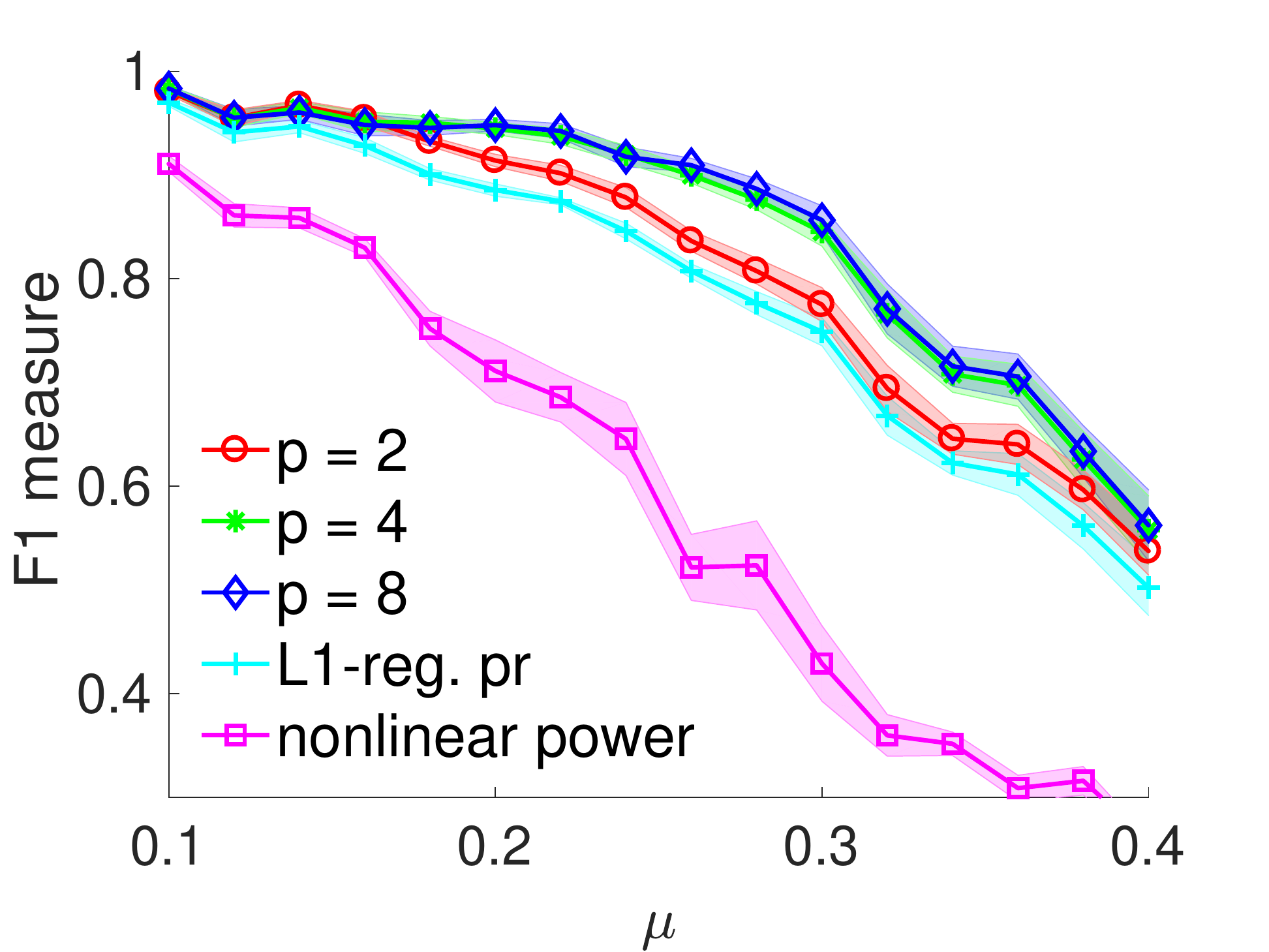}
	\end{subfigure}
	\caption{Conductance and F1 measure of various models on LFR synthetic datasets. The bands show the variation over 100 trials.}
    \label{fig:LFR}
\end{figure}

\subsubsection{Results on real social and biological networks}

For each dataset, we run algorithms starting from each node in each cluster and report the average conductance and F1 measure. We omit nonlinear power diffusion for Orkut, as it does not scale well to large graphs.
The Facebook datasets contain ground truth clusters ranging from low to medium conductance. For almost all clusters, $p$-norm flow diffusion methods have both the best F1 measure and conductance result. The six clusters in the biological dataset Sfld are very noisy, having median conductance around 0.8. Hence it is highly likely that a ground truth cluster is contained in some larger clusters (but not one of the 29 true families) that have much lower conductances. This partially explains why nonlinear power diffusion returns low conductance clusters but has poor recovery performance in terms of F1 measures. Note that the nonlinear diffusion model is the only global method that we compare with. Both $p$-norm flow diffusions and $\ell_1$-regularized PageRank are local methods, and they take advantage of locality to find local clusters that align well with the ground truth. Therefore, the results for the Sfld dataset demonstrate the advantage of local methods at recovering relatively small ground truth clusters. On Orkut dataset, $4$-norm flow diffusion gives best result on all clusters.

 \vspace{-2mm}
\begin{table}[htb!]
    \caption{Results for real-world graphs}
\vspace{2mm}
  \centering
    \begin{adjustbox}{max width=0.7\textwidth}
    \begin{tabular}{ccccccc}
    dataset & feature & measure & $p=2$ & $p=4$ & $\ell_1$-reg. pr & nonlinear \\
    \midrule
     \multirow{9}{*}{\rotatebox[origin=c]{90}{FB-Johns55}} 	& \multirow{2}{*}{year 2006}  	& f1 		& \bf 0.36 	& 0.35 	& 0.31	& 0.31 \\	
     											&						& cond 	& \bf 0.34	& \bf 0.34 	& 0.50	& 0.40  \\\cdashline{2-7}
											& \multirow{2}{*}{year 2007}  	& f1 		& \bf 0.39 	& \bf 0.39 	& 0.38	& 0.36 \\	
     											&						& cond 	& 0.31	& \bf 0.30 	& 0.45	& 0.41  \\\cdashline{2-7}
											& \multirow{2}{*}{year 2008}  	& f1 		& \bf 0.51 	& \bf 0.51 	& \bf 0.51	& 0.37 \\	
     											&						& cond 	& \bf 0.34	& \bf 0.34 	& 0.44	& 0.41  \\\cdashline{2-7}
											& \multirow{2}{*}{year 2009}  	& f1 		& 0.84 	& \bf 0.85 	& 0.83	& 0.49 \\	
     											&						& cond 	& 0.23	& \bf 0.22 	& 0.24	& 0.40  \\\cdashline{2-7}
											& \multirow{2}{*}{major index 217}  & f1 	& 0.85 	& \bf 0.87 	& 0.83	& 0.75 \\	
     											&						& cond 	& 0.23	& \bf 0.22 	& 0.25	& 0.29  \\\hdashline
     \multirow{12}{*}{\rotatebox[origin=c]{90}{Colgate88}}	& \multirow{2}{*}{year 2004}  	& f1 		& 0.50 	& \bf 0.51 	& 0.43	& 0.25 \\	
     											&						& cond 	& 0.66	& 0.66 	& 0.71	& \bf 0.36  \\\cdashline{2-7}
											& \multirow{2}{*}{year 2005}  	& f1 		& \bf 0.45 	& \bf 0.45 	& 0.41	& 0.37 \\	
     											&						& cond 	& 0.51	& 0.51 	& 0.53	& \bf 0.37  \\\cdashline{2-7}
											& \multirow{2}{*}{year 2006}  	& f1 		& \bf 0.45 	& \bf 0.45 	& 0.43	& 0.39 \\	
     											&						& cond 	& \bf 0.37	& \bf 0.36 	& 0.50	& 0.38  \\\cdashline{2-7}
											& \multirow{2}{*}{year 2007}  	& f1 		& 0.49 	& 0.49 	& \bf 0.51	& 0.45 \\	
     											&						& cond 	&\bf  0.34	& \bf 0.34 	& 0.45	& 0.39  \\\cdashline{2-7}
											& \multirow{2}{*}{year 2008}  	& f1 		& 0.76 	& \bf 0.80 	& 0.74	& 0.55 \\	
     											&						& cond 	& 0.31	& \bf 0.30 	& 0.35	& 0.40  \\\cdashline{2-7}
											& \multirow{2}{*}{year 2009}  	& f1 		& 0.96 	& \bf 0.97 	& 0.96	& 0.82 \\
     											&						& cond 	& 0.13	& \bf 0.12 	& 0.13	& 0.24  \\
    \midrule
    \multirow{12}{*}{\rotatebox[origin=c]{90}{Sfld}} 	& \multirow{2}{*}{urease}			& f1 		& 0.74	& \bf 0.76	& 0.72	& 0.63 \\	
     										&							& cond 	& \bf 0.44	& 0.45	& \bf 0.44	& 0.48 \\
										\cdashline{2-7}
										& \multirow{2}{*}{AMP 	}  		& f1 		& \bf 0.83	& \bf 0.83	& \bf 0.83	& \bf 0.83 \\	
     										&							& cond 	& \bf 0.41	& \bf 0.41	& 0.42	& 0.43 \\
										\cdashline{2-7}
										& \multirow{2}{*}{phosphotriesterase}& f1 		& 0.93	& 0.93	& \bf 1.00	& 0.13 \\	
     										&							& cond 	& 0.81	& 0.81	& 0.81	& \bf 0.49 \\
										\cdashline{2-7}
     										& \multirow{2}{*}{adenosine}  		& f1 		& \bf 0.44	& \bf 0.44	& \bf 0.44	& 0.34 \\	
     										&							& cond 	& 0.46	& 0.46	& 0.46	& \bf 0.44  \\
										\cdashline{2-7}
										& \multirow{2}{*}{dihydroorotase3}  	& f1 		& \bf 0.96	& \bf 0.96	& \bf 0.96	& 0.07 \\	
     										&							& cond 	& 0.84	& 0.84	& 0.84	& \bf 0.44 \\
										\cdashline{2-7}
										& \multirow{2}{*}{dihydroorotase2}  	& f1 		& \bf 0.39	& \bf 0.39	& 0.20	& 0.20 \\	
     										&							& cond 	& 0.77	& 0.78	& 0.85	& \bf 0.48\\
    \midrule
    \multirow{20}{*}{\rotatebox[origin=c]{90}{Orkut}} 	& \multirow{2}{*}{A}  	& f1 		& 0.56 	& \bf 0.58 	& 0.49 \\
     										&				& cond 	& 0.48	& \bf 0.47 	& 0.51 \\\cdashline{2-6}
     										& \multirow{2}{*}{B}  	& f1 		& 0.71 	& \bf 0.73 	& 0.66 \\
     										&				& cond 	& 0.35	& \bf 0.33 	& 0.37 \\\cdashline{2-6}
										& \multirow{2}{*}{C}  	& f1 		& 0.63 	& \bf 0.64 	& 0.57 \\
     										&				& cond 	& 0.33	& \bf 0.32 	& 0.35 \\\cdashline{2-6}
										& \multirow{2}{*}{D}  	& f1 		& 0.73 	& \bf 0.76 	& 0.72 \\
     										&				& cond 	& 0.49	& \bf 0.48 	& 0.51 \\\cdashline{2-6}
     										& \multirow{2}{*}{E}  	& f1 		& 0.61 	& \bf 0.62 	& 0.56 \\
     										&				& cond 	& 0.52	& \bf 0.51 	& 0.54 \\\cdashline{2-6}
										& \multirow{2}{*}{F}  	& f1 		& 0.79 	& \bf 0.81 	& 0.76 \\
     										&				& cond 	& 0.52	& \bf 0.51 	& 0.54 \\\cdashline{2-6}
										& \multirow{2}{*}{G}  	& f1 		& 0.72 	& \bf 0.73 	& 0.68 \\
     										&				& cond 	& 0.52	& \bf 0.50 	& 0.53 \\\cdashline{2-6}
										& \multirow{2}{*}{H}  	& f1 		& 0.68 	& \bf 0.70 	& 0.67 \\
     										&				& cond 	& 0.52	& \bf 0.51 	& 0.54 \\\cdashline{2-6}
										& \multirow{2}{*}{I}  	& f1 		& 0.60 	& \bf 0.62 	& 0.56 \\
     										&				& cond 	& 0.49	& \bf 0.48 	& 0.52 \\\cdashline{2-6}
										& \multirow{2}{*}{J}  	& f1 		& 0.52 	& \bf 0.54 	& 0.47 \\
     										&				& cond 	& 0.53	& \bf 0.52 	& 0.56 \\\cdashline{2-6}
										& \multirow{2}{*}{K}  	& f1 		& 0.54 	& \bf 0.56 	& 0.51 \\
     										&				& cond 	& 0.57	& \bf 0.56 	& 0.60 \\									
    \midrule
    \end{tabular}
    \end{adjustbox}
    \label{tab:realdataresult}
\end{table}

\section{Conclusion}
\label{sec:conclusion}

In this work we draw inspiration from spectral and combinatorial  methods for local graph clustering. We propose a new method that is naturally non-linear and strongly local, and offers a spectrum of clustering guarantees ranging from quadratic approximation error (typically obtained by spectral diffusions) to constant approximation error (typically obtained by combinatorial diffusions). We note that the proposed diffusion is different from previous methods, since we analyze the optimal solution of a clearly defined family of optimization problems, rather than the output of an algorithmic procedure with no optimization objective, e.g.\ methods based on random walk. Furthermore, we provide a strongly local algorithm that solves the optimization formulation efficiently, which enables our method to scale to real-world large graphs with billions of nodes and edges. 

We point out that a major advantage of this work is the simplicity of the proposed optimization formulation, clustering analysis, and algorithm design. Our model requires very few parameter tuning, and hence, it is extremely easy to use in practice and delivers consistent results. The algorithm is simple and has a very intuitive combinatorial interpretation, which facilitates possible future extension of $p$-norm flow diffusion to other applications and in much broader contexts, e.g., in community detection and graph semi-supervised learning~\citep{FLGM2020}, in defining localized network centrality measures~\citep{YSWBF2020}.

\newpage
\bibliography{ref}

\begin{thebibliography}{10}

\bibitem{Amghibech2003}
S.~Amghibech.
\newblock Eigenvalues of the discrete p-laplacian for graphs.
\newblock {\em Ars Comb.}, 67:283--302, April 2003.

\bibitem{ACL06}
R.~Andersen, F.~Chung, and K.~Lang.
\newblock Local graph partitioning using pagerank vectors.
\newblock {\em FOCS '06 Proceedings of the 47th Annual IEEE Symposium on
  Foundations of Computer Science}, pages 475--486, 2006.

\bibitem{AGPT2016}
R.~Andersen, S.~O. Gharan, Y.~Peres, and L.~Trevisan.
\newblock Almost optimal local graph clustering using evolving sets.
\newblock {\em Journal of the ACM}, 63(2), 2016.

\bibitem{AL08_SODA}
R.~Andersen and K.~J. Lang.
\newblock An algorithm for improving graph partitions.
\newblock {\em Proceedings of the nineteenth annual ACM-SIAM symposium on
  Discrete algorithms}, pages 651--660, 2008.

\bibitem{AP09}
Reid Andersen and Yuval Peres.
\newblock Finding sparse cuts locally using evolving sets.
\newblock In {\em {STOC} 2009}, pages 235--244, 2009.

\bibitem{brown2006gold}
Shoshana~D Brown, John~A Gerlt, Jennifer~L Seffernick, and Patricia~C Babbitt.
\newblock A gold standard set of mechanistically diverse enzyme superfamilies.
\newblock {\em Genome biology}, 7(1):R8, 2006.

\bibitem{BM2009}
T.~B\"{u}hler and M.~Hein.
\newblock Spectral clustering based on the graph p-laplacian.
\newblock In {\em Proceedings of the 26th Annual International Conference on
  Machine Learning}, ICML 2009, pages 81--88, New York, NY, USA, 2009.
  Association for Computing Machinery.

\bibitem{chung2007-pagerank-heat}
F.~Chung.
\newblock The heat kernel as the {PageRank} of a graph.
\newblock {\em Proceedings of the National Academy of Sciences},
  104(50):19735--19740, 2007.

\bibitem{C09}
F.~R.~K. Chung.
\newblock A local graph partitioning algorithm using heat kernel pagerank.
\newblock {\em Internet Mathematics}, 6(3):315--330, 2009.

\bibitem{CHUNG200722}
Fan Chung.
\newblock Random walks and local cuts in graphs.
\newblock {\em Linear Algebra and its Applications}, 423(1):22--32, 2007.
\newblock Special Issue devoted to papers presented at the Aveiro Workshop on
  Graph Spectra.

\bibitem{FM16}
S.~Fortunato and D.~Hric.
\newblock Community detection in networks: A user guide.
\newblock {\em Physics Reports}, 659(11):1--44, 2016.

\bibitem{FDM2017}
K.~Fountoulakis, D.~F. Gleich, and M.~W. Mahoney.
\newblock An optimization approach to locally-biased graph algorithms.
\newblock {\em Proceedings of the IEEE}, 105(2):256--272, 2017.

\bibitem{FLGM2020}
K.~Fountoulakis, M.~Liu, D.~F. Gleich, and M.~W. Mahoney.
\newblock Flow-based algorithms for improving clusters: A unifying framework,
  software, and performance.
\newblock arXiv:2004.09608 [cs.LG], 2020.

\bibitem{FKSCM2017}
K.~Fountoulakis, F.~Roosta-Khorasani, J.~Shun, X.~Cheng, and M.~W. Mahoney.
\newblock Variational perspective on local graph clustering.
\newblock {\em Mathematical Programming}, 174:553--573, 2017.

\bibitem{Gle15_SIREV}
D.~F. Gleich.
\newblock {PageRank} beyond the web.
\newblock {\em SIAM Review}, 57(3):321--363, 2015.

\bibitem{cvx}
Michael Grant and Stephen Boyd.
\newblock {CVX}: Matlab software for disciplined convex programming, version
  2.1.
\newblock \url{http://cvxr.com/cvx}, March 2014.

\bibitem{guatterymiller98}
S.~Guattery and G.L. Miller.
\newblock On the quality of spectral separators.
\newblock {\em SIAM Journal on Matrix Analysis and Applications}, 19:701--719,
  1998.

\bibitem{ID19}
R.~Ibrahim and D.~Gleich.
\newblock Nonlinear diffusion for community detection and semi-supervised
  learning.
\newblock {\em {WWW}'19: The World Wide Web Conference}, pages 739--750, 2019.

\bibitem{Jeub15}
L.~G.~S. Jeub, P.~Balachandran, M.~A. Porter, P.~J. Mucha, and M.~W. Mahoney.
\newblock Think locally, act locally: Detection of small, medium-sized, and
  large communities in large networks.
\newblock {\em Physical Review E}, 91:012821, 2015.

\bibitem{KG14}
K.~Kloster and D.~F. Gleich.
\newblock Heat kernel based community detection.
\newblock {\em Proceedings of the 20th ACM SIGKDD international conference on
  Knowledge discovery and data mining}, pages 1386--1395, 2014.

\bibitem{KK2014}
I.~M. Kloumann and J.~M. Kleinberg.
\newblock Community membership identification from small seed sets.
\newblock {\em Proceedings of the 20th ACM SIGKDD international conference on
  Knowledge discovery and data mining}, pages 1366--1375, 2014.

\bibitem{LFR08}
Andrea Lancichinetti, Santo Fortunato, and Filippo Radicchi.
\newblock Benchmark graphs for testing community detection algorithms.
\newblock {\em Physical review. E, Statistical, nonlinear, and soft matter
  physics}, 78:046110, 11 2008.

\bibitem{LW2018}
C.~Lee and S.~J. Wright.
\newblock {Random permutations fix a worst case for cyclic coordinate descent}.
\newblock {\em IMA Journal of Numerical Analysis}, 39(3):1246--1275, July 2018.

\bibitem{LLDM09_communities_IM}
J.~Leskovec, K.J. Lang, A.~Dasgupta, and M.W. Mahoney.
\newblock Community structure in large networks: Natural cluster sizes and the
  absence of large well-defined clusters.
\newblock {\em Internet Mathematics}, 6(1):29--123, 2009.

\bibitem{snapnets}
Jure Leskovec and Andrej Krevl.
\newblock {SNAP Datasets}: {Stanford} large network dataset collection.
\newblock \url{http://snap.stanford.edu/data}, June 2014.

\bibitem{LM2018a}
P.~Li and O.~Milenkovic.
\newblock Submodular hypergraphs: p-laplacians, {C}heeger inequalities and
  spectral clustering.
\newblock In Jennifer Dy and Andreas Krause, editors, {\em Proceedings of the
  35th International Conference on Machine Learning}, volume~80 of {\em
  Proceedings of Machine Learning Research}, pages 3014--3023,
  Stockholmsm\"{s}san, Stockholm Sweden, 10--15 Jul 2018. PMLR.

\bibitem{LS90}
L.~Lov{\'{a}}sz and M.~Simonovits.
\newblock The mixing rate of {M}arkov chains, an isoperimetric inequality, and
  computing the volume.
\newblock {\em Proceedings [1990] 31st Annual Symposium on Foundations of
  Computer Science}, pages 346--354, 1990.

\bibitem{LS93}
L{\'{a}}szl{\'{o}} Lov{\'{a}}sz and Mikl{\'{o}}s Simonovits.
\newblock Random walks in a convex body and an improved volume algorithm.
\newblock {\em Random Struct. Algorithms}, 4(4):359--412, 1993.

\bibitem{MOV12_JMLR}
M.~W. Mahoney, L.~Orecchia, and N.~K. Vishnoi.
\newblock A local spectral method for graphs: with applications to improving
  graph partitions and exploring data graphs locally.
\newblock {\em Journal of Machine Learning Research}, 13:2339--2365, 2012.

\bibitem{NJW01_spectral}
A.Y. Ng, M.I. Jordan, and Y.~Weiss.
\newblock On spectral clustering: Analysis and an algorithm.
\newblock In {\em NIPS '01: Proceedings of the 15th Annual Conference on
  Advances in Neural Information Processing Systems}, 2001.

\bibitem{OZ14}
L.~Orecchia and Z.~A. Zhu.
\newblock Flow-based algorithms for local graph clustering.
\newblock In {\em Proceedings of the 25th Annual ACM-SIAM Symposium on Discrete
  Algorithms}, pages 1267--1286, 2014.

\bibitem{PB99}
Lawrence Page, Sergey Brin, Rajeev Motwani, and Terry Winograd.
\newblock The pagerank citation ranking: Bringing order to the web.
\newblock Technical report, Stanford InfoLab, 1999.

\bibitem{2017-ecml-pkdd}
P.~Shi, K.~He, D.~Bindel, and J.~Hopcroft.
\newblock Local {L}anczos spectral approximation for community detection.
\newblock In {\em Proceedings of ECML-PKDD}, September 2017.

\bibitem{ST13}
D.~A. Spielman and S.~H. Teng.
\newblock A local clustering algorithm for massive graphs and its application
  to nearly linear time graph partitioning.
\newblock {\em SIAM Journal on Scientific Computing}, 42(1):1--26, 2013.

\bibitem{TMP2012}
A.~L. Traud, P.~J. Mucha, and M.~A. Porter.
\newblock Social structure of facebook networks.
\newblock {\em Physica A: Statistical Mechanics and its Applications},
  391(16):4165--4180, 2012.

\bibitem{Tsiatas12}
A.~Tsiatas.
\newblock Diffusion and clustering on large graphs.
\newblock {\em PhD Thesis, University of California, San Diego}, 2012.

\bibitem{Tuncbag-2016-glioblastoma}
N.~Tuncbag, P.~Milani, J.~L. Pokorny, H.~Johnson, T.~T. Sio, S.~Dalin, D.~O.
  Iyekegbe, F.~M. White, J.~N. Sarkaria, and E.~Fraenkel.
\newblock Network modeling identifies patient-specific pathways in
  glioblastoma.
\newblock {\em Scientific Reports}, 6:28668, 2016.

\bibitem{luxburg05_survey}
U.~von Luxburg.
\newblock A tutorial on spectral clustering.
\newblock Technical Report 149, Max Plank Institute for Biological Cybernetics,
  August 2006.

\bibitem{WFHM2017}
D.~Wang, K.~Fountoulakis, M.~Henzinger, M.~W. Mahoney, and S.~Rao.
\newblock Capacity releasing diffusion for speed and locality.
\newblock {\em Proceedings of the 34th International Conference on Machine
  Learning}, 70:3598–3607, 2017.

\bibitem{WS05_spectralSDM}
S.~White and P.~Smyth.
\newblock A spectral clustering approach to finding communities in graphs.
\newblock In {\em SDM '05: Proceedings of the 5th SIAM International Conference
  on Data Mining}, pages 76--84, 2005.

\bibitem{YSWBF2020}
S.~Yang, P.~Senapati, D.~Wang, C.~T. Bauch, and K.~Fountoulakis.
\newblock Targeted pandemic containment through identifying local contact
  network bottlenecks.
\newblock arXiv:2006.06939 [cs.SI], 2020.

\bibitem{ALM13}
Zeyuan~Allen Zhu, Silvio Lattanzi, and Vahab~S. Mirrokni.
\newblock A local algorithm for finding well-connected clusters.
\newblock In {\em {ICML} 2013}, pages 396--404, 2013.

\end{thebibliography}
\bibliographystyle{plain}

\end{document}